\documentclass[11pt]{article}
 
\usepackage[sort,numbers]{natbib}

\usepackage[bookmarks=true,bookmarksnumbered=true,colorlinks=true,allcolors=black]{hyperref}
\usepackage{amsmath,amssymb,amsfonts,amsthm}
\usepackage{lscape}
\usepackage{arydshln}
\renewcommand*{\baselinestretch}{1.25}

\usepackage[capitalize]{cleveref}

\usepackage{geometry}
\usepackage{indentfirst}

\usepackage{enumitem}
\usepackage{color}

\newtheorem{theorem}{Theorem}[section]
\newtheorem{lemma}{Lemma}[section]
\newtheorem{proposition}{Proposition}[section]
\newtheorem{corollary}{Corollary}[section]
\theoremstyle{definition}
\newtheorem{definition}{Definition}[section]
\newtheorem*{rmk*}{Remark}
\newtheorem{rmk}{Remark}[section]

\newtheorem{example}{Example}[section]

\DeclareMathOperator{\Var}{Var}
\DeclareMathOperator{\Cov}{Cov}

\DeclareMathOperator{\trace}{Tr}

\DeclareMathOperator{\E}{E}

\DeclareMathOperator{\kl}{KL}

\DeclareMathOperator{\wass}{W}
\DeclareMathOperator{\diver}{div}

\allowdisplaybreaks

\numberwithin{equation}{section}

\makeatletter
    \renewcommand*{\section}{\@startsection{section}{1}{\z@}%
    {6pt}{3pt}{\reset@font\normalsize\bfseries}}
\makeatother
    
\makeatletter
    \renewcommand*{\subsection}{\@startsection{subsection}{2}{\z@}%
    {3pt}{3pt}{\reset@font\normalsize\mdseries\itshape}}
\makeatother

\makeatletter
    \renewcommand*{\subsubsection}{\@startsection{subsubsection}{3}{\z@}%
    {3pt}{3pt}{\reset@font\normalsize\mdseries\itshape}}
\makeatother

\makeatletter
\def\@listi{\leftmargin\leftmargini
  \topsep=.5\baselineskip 
  \partopsep=0pt \parsep=0pt \itemsep=0pt}
\let\@listI\@listi
\@listi
\def\@listii{\leftmargin\leftmarginii
  \labelwidth\leftmarginii \advance\labelwidth-\labelsep
  \topsep=0pt \partopsep=0pt \parsep=0pt \itemsep=0pt}
\def\@listiii{\leftmargin\leftmarginiii
  \labelwidth\leftmarginiii \advance\labelwidth-\labelsep
  \topsep=0pt \partopsep=0pt \parsep=0pt \itemsep=0pt}
\def\@listiv{\leftmargin\leftmarginiv
  \labelwidth\leftmarginiv \advance\labelwidth-\labelsep
  \topsep=0pt \partopsep=0pt \parsep=0pt \itemsep=0pt}
\makeatother

\makeatletter
\newcommand{\opnorm}{\@ifstar\@opnorms\@opnorm}
\newcommand{\@opnorms}[1]{%
  \left|\mkern-1.5mu\left|\mkern-1.5mu\left|
   #1
  \right|\mkern-1.5mu\right|\mkern-1.5mu\right|
}
\newcommand{\@opnorm}[2][]{%
  \mathopen{#1|\mkern-1.5mu#1|\mkern-1.5mu#1|}
  #2
  \mathclose{#1|\mkern-1.5mu#1|\mkern-1.5mu#1|}
}
\makeatother

\def\be#1{\begin{equation*}#1\end{equation*}}
\def\ben#1{\begin{equation}#1\end{equation}}

\def\besn#1{\begin{equation}\begin{split}#1\end{split}\end{equation}}

\def\bm#1{\begin{multline*}#1\end{multline*}}
\def\bmn#1{\begin{multline}#1\end{multline}}
\def\ba#1{\begin{align*}#1\end{align*}}
\def\ban#1{\begin{align}#1\end{align}}

\newcommand{\eps}{\varepsilon}

\newcommand{\scored}{\mathsf{s}}

\newcommand{\slepian}{\mathcal{S}}
\newcommand{\euler}{\hat{X}}
\newcommand{\eul}{X^N}

\newcommand{\ddpm}{\mathsf{Y}}
\newcommand{\ddpmy}{\tilde{\ddpm}}
\newcommand{\ddpmt}{\ddpm^*}

\newcommand{\inte}{\,\mathrm{d}}
\newcommand{\mean}{\hat{\mu}}
\newcommand{\target}{P^*}
\newcommand{\indices}{\mathcal{I}_N}

\newcommand{\mcl}[1]{\mathcal{#1}}
\newcommand{\norm}[1]{\left\|#1\right\|}
\newcommand{\bra}[1]{\left(#1\right)}
\newcommand{\cbra}[1]{\left\{#1\right\}}
\newcommand{\sbra}[1]{\left[#1\right]}

\newcommand{\abs}[1]{\left|#1\right|}

\makeatletter
\newcommand{\pushright}[1]{\ifmeasuring@#1\else\omit\hfill$\displaystyle#1$\fi\ignorespaces}
\newcommand{\pushleft}[1]{\ifmeasuring@#1\else\omit$\displaystyle#1$\hfill\fi\ignorespaces}
\makeatother

\title{Wasserstein bounds for denoising diffusion probabilistic models via the F\"ollmer process}
\author{Yuta Koike
\thanks{Graduate School of Mathematical Sciences, University of Tokyo}
\thanks{CREST, Japan Science and Technology Agency}}

\begin{document}

\maketitle

\begin{abstract} 
This paper studies sampling error bounds for denoising diffusion probabilistic models (DDPMs) in the 2-Wasserstein distance. 
Our contributions are threefold. 
(i) Under general Lipschitz-type conditions on the score function and for a broad class of variance schedules, including the cosine schedule, we establish sharp upper bounds that are optimal in both the dimension and the number of steps, and recover several sharp error bounds previously obtained in the literature. 
(ii) We prove that the same Lipschitz-type conditions, which encompass those commonly imposed on the (learned) score, imply a logarithmic Sobolev inequality and hence a quadratic transportation cost inequality for the DDPM. 
As a consequence, in settings covered by existing work, an optimal Wasserstein bound, up to a logarithmic factor, follows from the recently obtained sharp error bound in the Kullback--Leibler divergence under geometric-type variance schedules. 
(iii) We show that for general log-concave target distributions, the optimal Wasserstein error bound remains attainable even without a quadratic transportation cost inequality for the target. 
Our analysis is based on viewing the DDPM sampler as a discretization of the F\"ollmer process rather than the conventional reverse Ornstein--Uhlenbeck process. 
\vspace{2mm}

\noindent \textit{Keywords}: 
cosine schedule, 
initialization, 
log-concave distribution, 
quadratic transportation cost inequality,
sampling error, 
stochastic localization.

\end{abstract}

\section{Introduction}

Over the last decade, we have witnessed dramatic progress in generative artificial intelligence (AI) and its successful applications across a wide range of fields. 
This paper focuses on generative models, the mathematical foundation of generative AI. 
Mathematically speaking, generative models aim to construct a sampling procedure for an unknown probability distribution $\target$ given training data drawn from $\target$. 
\emph{Score-based generative models} (SGMs), also known as \emph{diffusion models}, construct such a sampler by estimating the \emph{score function} of $\target$. 
Owing to their high sample quality and diversity, SGMs are now widely regarded as a leading class of generative models for various tasks, particularly in image generation \cite{cao2024survey,yang2023diffusion,croitoru2023diffusion}. 
At the same time, sampling from a trained SGM is computationally more expensive than competing approaches such as generative adversarial networks (GANs) and variational autoencoders (VAEs), since it requires iterative evaluations of the learned score function \cite{XKV22}. 
This computational bottleneck naturally motivates the study of sampling errors for SGMs. 

In this paper, we develop sampling error bounds for the \emph{denoising diffusion probabilistic model} (DDPM) proposed in \cite{SDWMG15,HJA20}, which is one of the most widely used SGMs. 
Previous works have mainly analyzed sampling errors for the DDPM in terms of total variation (TV) distance and Kullback--Leibler (KL) divergence \cite{BDDD24,CDGS24,CLL23,LWCC24,HWC26,JaZh26,JZL25,Ko26ddpm,LiYa24,LiYa25,LHC25}, yielding sharp bounds under fairly general assumptions on $\target$. 
By contrast, bounds in the 2-Wasserstein distance have recently attracted increasing attention due to their statistical relevance (see \cite[Section 3.4]{AVD25} and \cite[Section 1.2]{BeBa25}). 
As discussed in \cite[Section 3.5]{AVD25}, analyzing sampling errors in Wasserstein distance is inherently more challenging than in TV distance or KL divergence. 
Although there is a rapidly growing body of work in this direction \cite{BrSa25,GNZ25,GSO25,SOBLCL25,YuYu25,AVD25,St26,WaWa25,BeBa25,PTW26}, the problem remains far from being fully understood. We refer to Table 1 in \cite{St26} for a comprehensive summary of existing results. 

The aim of this paper is to advance this line of research by providing new theoretical insights into sampling error bounds in Wasserstein distance for the DDPM. 
Specifically, our contributions are threefold.  
First, we establish sharp upper bounds for the 2-Wasserstein distance between the DDPM sample and $\target$ under general Lipschitz-type conditions on the score function, and for a wide class of variance schedules, including the cosine schedule of \cite{NiDh21} (Theorems \ref{thm:main1}--\ref{thm:main3}). 
As a benchmark case, we illustrate our bound when $\target$ is weakly log-concave in the sense of \cite{GSO25,KITL25} (Corollaries \ref{coro1} and \ref{coro2}), demonstrating that it is optimal in both the dimension and the number of steps, and is tight in the sense that it vanishes for the standard normal distribution. 
Our results also recover several sharp error bounds recently obtained in the literature \cite{AVD25,WaWa25,St26}. 
Second, we show that the DDPM satisfies a logarithmic Sobolev inequality when the learned score function satisfies a Lipschitz-type condition (\cref{lem:LS}). 
An important consequence of this result is that a 2-Wasserstein bound for the DDPM can be derived from the corresponding KL bound under a Lipschitz-type condition on the (learned) score, via the celebrated quadratic transportation cost inequality of \citet{OtVi00}. 
In particular, existing optimal 2-Wasserstein bounds for the DDPM can essentially be deduced from the recent sharp KL bound of \cite{JZL25} under geometric-type variance schedules. 
This raises the question of whether one can still achieve the optimal Wasserstein error bound even when $\target$ does not satisfy a quadratic transportation cost inequality. 
Third, we provide an affirmative answer to this question in the general log-concave setting; see Theorems \ref{prop:lc} and \ref{thm:lc}.  

Unlike the conventional approach, our analysis is based on interpreting the DDPM sampler as a discretization of the \emph{F\"ollmer process}, rather than the reverse Ornstein--Uhlenbeck (OU) process. 
See \cref{sec:follmer} for the definition and key properties of the F\"ollmer process, and \cite{Ko26ddpm} for a more detailed account. 
This perspective offers several technical advantages. 
First, the drift process of the F\"ollmer process is a martingale, which simplifies the bias-variance decomposition of the sampling error compared to existing approaches; in particular, the bias term vanishes automatically. 
Second, the associated discretization step sizes are effectively smaller than those in the reverse OU formulation, which allows us to consider variance schedules beyond the scope of standard theoretical analyses. 
Third, discretizating the F\"ollmer process enables a more precise analysis of the initialization error in the DDPM. 
In particular, our results clarify that the performance of the DDPM sampler is improved by appropriately choosing its initial mean.

The rest of the paper is organized as follows. 
\cref{sec:setting} provides a formal description of the DDPM and related concepts. 
\cref{sec:main} presents the main results of the paper. 
\cref{sec:proof} demonstrates proofs of the main results, while \cref{sec:appendix} contains additional proofs for auxiliary results. 

\paragraph{Notation}

For an integer $N\geq1$, we write $\indices=\{0,1,\dots,N-1\}$ for short. 
For a vector $x\in\mathbb R^d$, we denote its Euclidean norm by $|x|$. 
For a matrix $A$, its operator norm and Frobenius norm are denoted by $\|A\|_{op}$ and $\|A\|_F$, respectively. 
For two matrices $A$ and $B$ of the same size, we write $A\preceq B$ or $B\succeq A$ if $B-A$ is symmetric and positive semidefinite. 

For $r\in\mathbb N$, $(\mathbb R^d)^{\otimes r}$ denotes the set of real-valued $d$-dimensional $r$-arrays $V=(V_{j_1,\dots,j_r})_{1\leq j_1,\dots,j_r\leq d}$. 
In particular, $(\mathbb R^d)^{\otimes1}=\mathbb R^d$ and $(\mathbb R^d)^{\otimes2}$ is the set of $d\times d$ matrices. 
For $U\in(\mathbb R^d)^{\otimes q}$ and $V\in(\mathbb R^d)^{\otimes r}$, we set $U\otimes V:=(U_{i_1,\dots,i_q}V_{j_1,\dots,j_r})_{1\leq i_1,\dots,i_q,j_1,\dots,j_r\leq d}\in(\mathbb R^{d})^{\otimes(q+r)}$, and write $U^{\otimes2}=U\otimes U$ for short. 
When $q=r$, we also set $\langle U,V\rangle:=\sum_{j_1,\dots,j_r=1}^dU_{j_1,\dots,j_r}V_{j_1,\dots,j_r}$. In particular, when $q=r=1$, $\langle U,V\rangle$ is the Euclidean inner product of $U$ and $V$, which we also write as $U\cdot V$. 

If $\xi$ is a random vector in $\mathbb R^d$, we write $P^\xi$ for its law. 
For $p\geq1$, we set $\|\xi\|_p:=(\E[|\xi|^p])^{1/p}$. 
For a stochastic process $Y=(Y_t)_{t\in I}$ in $\mathbb R^d$, we denote by $Y^i_t$ the $i$-th component of $Y_t$ for $t\in I$ and $i=1,\dots,d$. 
We write $\phi_d$ for the $d$-dimensional standard normal density. 


For two probability distributions $P$ and $Q$ on $\mathbb R^d$, the 2-Wasserstein distance is defined by
\[
\wass_2(P,Q):=\inf_{\pi\in\Pi(P,Q)}\bra{\int_{\mathbb R^d}|x-y|^2\pi(\mathrm{d}x\mathrm{d} y)}^{1/2},
\]
where $\Pi(P,Q)$ denotes the set of all couplings of $P$ and $Q$, i.e., probability distributions $\pi$ on $\mathbb R^d\times\mathbb R^d$ such that $P(A)=\pi(A\times\mathbb R^d)$ and $Q(A)=\pi(\mathbb R^d\times A)$ for all Borel sets $A\subset\mathbb R^d$. 
When $\xi$ is a random vector in $\mathbb R^d$, we write $\wass_2(\xi,Q)=\wass_2(P^\xi,Q)$ for short. 
Also, the KL divergence of $P$ with respect to $Q$ is defined by
\[
\kl(P\mid Q)=\int_{\mathbb R^d}\frac{\inte P}{\inte Q}\log\bra{\frac{\inte P}{\inte Q}}\inte Q\quad\text{if }P\ll Q
\]
and $\kl(P\mid Q)=\infty$ otherwise.

\section{Setting}\label{sec:setting}

Let $\target$ be a probability distribution on $\mathbb R^d$. 
Throughout the paper, we assume that $\target$ has at least a finite second moment, and we denote its mean and covariance matrix by $\mu$ and $\Sigma$, respectively:
\[
\mu:=\int_{\mathbb R^d}x\target(\mathrm{d}x),
\qquad
\Sigma:=\int_{\mathbb R^d}(x-\mu)^{\otimes2}\target(\mathrm{d}x).
\] 
For $t\in[0,1)$, we denote by $\slepian_t\target$ the law of the random vector $\sqrt t\xi+\sqrt{1-t}Z$, where $\xi\sim\target$ and $Z\sim N(0,I_d)$ are independent. 
$\slepian_t\target$ admits a smooth and strictly positive density $\pi_t$, given by
\ba{
\pi_t(y)=\frac{1}{(1-t)^{d/2}}\int\phi_d\bra{\frac{y-\sqrt t x}{\sqrt{1-t}}}\target(\mathrm{d} x),\quad y\in\mathbb R^d.
}
The function $\nabla\log\pi_t$ is called the \emph{(Stein) score function} of $\target$ at time $t$. 

Given constants $0<\beta_i<1$ for $i\in\indices$, suppose that we have a proxy $\scored_i(x)$ of $\nabla\log\pi_{t_i}(x)$, where 
\ben{\label{def:time}
t_i:=(1-\delta)\prod_{j=i}^{N-1}(1-\beta_j)\quad\text{for }i=0,1,\dots,N-1,
}
and $\delta\in[0,1)$ is a constant. 
The case $\delta>0$ corresponds to performing \emph{early stopping}. 
We set $t_N:=1-\delta$ by convention. 
The DDPM sampler generates a sequence $(\ddpm_i)_{i=0}^N$ of random vectors in $\mathbb R^d$ by the following recursion relation:
\ben{\label{eq:ddpm}
\begin{cases}
\ddpm_0\sim N(\sqrt{t_0}\mean,I_d),\\
\ddpm_{i+1}=\frac{1}{\sqrt{1-\beta_{i}}}\bra{\ddpm_{i}+\beta_i\scored_i(\ddpm_i)}+\sqrt{\beta_i}Z_i\quad(i=0,1,\dots,N-1),
\end{cases}
}
where $\mean\in\mathbb R^d$ is a prescribed initialization vector and $Z_i\overset{i.i.d.}{\sim}N(0,I_d)$. 
$\mean$ is usually set to 0, but we will later see that choosing $\mean$ close to $\mu$ reduces initialization errors. 
We also note that our indexing convention is the reverse of the usual one. 


When $\mean=0$, the sampler \eqref{eq:ddpm} coincides with the original sampling method of \cite{SDWMG15,HJA20} with the hyper-parameters optimized for the case $\target=N(0,I_d)$. 
\citet{SSKKEP21} refer to this procedure as \emph{ancestral sampling}. 
They also propose an alternative method, called \emph{reverse diffusion sampling}, which constructs a sampler via discretization of the reverse stochastic differential equation (SDE) in their SDE formulation of SGMs. 
For the DDPM, these two samplers become asymptotically equivalent when $\max_{i\in\indices}\beta_i\to0$ (see \cite[Appendix E]{SSKKEP21} and \cite[Appendix B]{DeB22}). 
In this interpretation, $\beta_i$ approximately corresponds to the $i$-th step size of the discretization.  
%
By contrast, in our formulation the quantities
\[
h_i:=t_{i+1}-t_i,\quad i=0,1,\dots,N-1,
\] 
serve as the effective step sizes. 
In particular, we do not necessarily require $\max_{i\in\indices}\beta_i$ to be small, which allows us to consider variance schedules that lie beyond the scope of standard theoretical analyses in the literature. See Remarks \ref{rmk:sch1} and \ref{rmk:sch2} for further discussions. 

\subsection{Variance schedule}

The sequence $(\beta_i)_{i=0}^{N-1}$ is called the \emph{variance schedule} or the \emph{noise schedule} of the DDPM. 
On the one hand, it determines the variances of the Gaussian noise added to the data in the forward process.  
On the other hand, as discussed above, in the context of reverse diffusion sampling, it corresponds approximately to the step-size sequence of the discretization. 
The following two variance schedules are commonly used in theoretical analyses of the DDPM. 
\begin{example}[Constant schedule]\label{ex:const}
We set $\beta_i=\beta_0$ for all $i\in\indices$. 
This variance schedule is used when the target distribution is sufficiently smooth and early stopping is not required; see, e.g., \cite{CLL23,CDGS24}.
\end{example}

\begin{example}[Geometric schedule]\label{ex:geo}
Given positive constants $c_0$ and $c_1$, we set
\ben{\label{eq:geo}
\beta_0=\frac{1}{N^{c_0}},\qquad
\beta_i=\frac{c_1\log N}{N}\min\cbra{\beta_0\bra{1+\frac{c_1\log N}{N}}^{N-i},\,1}\quad(i=1,\dots,N-1).
}
This type of variance schedule is commonly used when deriving TV and KL error bounds in general settings with early stopping; see, e.g., \cite{BDDD24,HWC26,CLL23,LWCC24,JZL25,LiYa24,LiYa25,LHC25}. 
\end{example}


Given a variance schedule $(\beta_i)_{i=0}^{N-1}$, the sequence $(t_i)_{i=0}^{N-1}$ in \eqref{def:time} is completely determined by the following relation (recall $t_N=1-\delta$ by convention):
\ben{\label{eq:beta-time}
\beta_i=1-t_i/t_{i+1},\qquad i=0,1,\dots,N-1.
}
Conversely, given an increasing sequence $0<t_0<t_1<\cdots<t_N=1-\delta$, the variance schedule defined by \eqref{eq:beta-time} gives the relation \eqref{def:time}. 
The next example is introduced in this way. 


\begin{example}[Cosine schedule \cite{NiDh21}]\label{ex:cos}
Given a constant $s\geq0$, let\footnote{Our definition uses the sine function rather than the cosine function simply because our indexing convention is the reverse of that in \cite{NiDh21}.} 
\[
t_i:=\sin^2(i\eta_N)\quad\text{for }i=1,\dots,N,
\quad\text{where }\eta_N:=\frac{\pi}{2N(1+s)}.
\]
Also, let $t_0\in(0,t_1)$. 
We then define $(\beta_i)_{i=0}^{N-1}$ by \eqref{eq:beta-time}. 
\end{example}

We refer to \cite{GLHGD25} for other examples of variance schedules.

\subsection{Assumptions on the learned score}\label{sec:score}

To control the approximation error of the learned score, we mainly impose the following assumption:  
\begin{enumerate}[label={\normalfont[H]}]

\item\label{ass:score}
There exists a non-decreasing function $\eps_\text{score}:(0,1)\to[0,\infty)$ such that 
\ben{\label{score-bound}
\norm{\nabla\log \pi_{t_i}(\ddpm_{i})-\scored_i(\ddpm_{i})}_2\leq\eps_\text{score}(t_i)\quad\text{for all }i\in\indices.
}

\end{enumerate}
%
Note that the $L^2$-error on the left-hand side of \eqref{score-bound} is evaluated with respect to the law of $\ddpm_i$, which itself depends on $\scored_i$. 
While this type of assumption is commonly used in the analysis of 2-Wasserstein bounds (e.g., \cite{GNZ25,GSO25,SOBLCL25,YuYu25,KITL25,AVD25}), it has been pointed out that it does not directly reflect the objective minimized during score training; see \cite[Section 4]{BeBa25} for instance. 
Indeed, the standard score-learning procedure implicitly minimizes $\norm{\nabla\log \pi_{t_i}(\xi)-\scored_i(\xi)}_2$ with $\xi\sim\slepian_{t_i}\target$ for every $i\in\indices$; see, e.g., \cite{OAS23}. 
For this reason, some works instead impose the following alternative assumption \cite{BeBa25,PTW26,St26,WaWa25}: 
\begin{enumerate}[label={\normalfont[H']}]

\item\label{ass:score-p}
There exists a non-decreasing function $\eps_\text{score}:(0,1)\to[0,\infty)$ such that 
\ben{\label{score-error-p}
\norm{\nabla\log \pi_{t_i}(\xi)-\scored_i(\xi)}_{2}\leq\eps_\text{score}(t_i)
\text{ with }\xi\sim\slepian_{t_i}\target
\quad\text{for all }i\in\indices.
} 

\end{enumerate}
When working with \ref{ass:score-p} instead of \ref{ass:score}, one typically needs to impose additional regularity assumptions on the learned score $\scored_i$; see the references cited above. 
For illustration, we will give a result of this type in \cref{thm:main3}.

We also note that $\eps_\text{score}$ is often assumed to be constant. 
Following \cite{St26}, we allow it to depend on time. 
This is motivated by the fact that $\eps_{\mathrm{score}}(t)$ is typically expected to scale as $(1-t)^{-1/2}$ according to \cite[Theorem 3.1]{OAS23} (see also \cite[Corollaries 4 and 5]{St25}).

\section{Main results}\label{sec:main}


For $t\in[0,1)$, we denote by $f_t$ the density of $\slepian_t\target$ with respect to $N(0,I_d)$, i.e., $f_t(y)=\phi_d(y)^{-1}\pi_t(y)$ for $y\in\mathbb R^d$. 
We first present a bound under a (two-sided) Lipschitz condition on $\nabla\log f_t$ without early stopping. 
\begin{theorem}\label{thm:main1}
Assume \ref{ass:score}, $\delta=0$ and $\max_{i\in\indices}\beta_i\leq3/4$. 
Suppose that there exists a non-decreasing function $L:(0,1)\to[0,\infty)$ such that 
\ben{\label{score-lip}
-tL(t)I_d\preceq\nabla^2\log f_t(x)\preceq tL(t)I_d
\quad\text{for all }0<t<1\text{ and }x\in\mathbb R^d. 
}
Suppose also that there exists a constant $\eta>0$ such that $h_i\leq\eta$ for all $i\in\indices$. 
Then,
\ba{
&\wass_2(\ddpm_N,\target)\\
&\leq e^{\int_{t_0}^{t_N}L(s)\inte s}\cbra{t_0\bra{|\mu-\mean|+\sqrt{d\int_0^{t_0}L(s)^2\inte s}}+\sqrt d\eta\sqrt{\int_{t_0}^{t_N}L(s)^2\inte s}+2\int_{t_0}^{t_N}\frac{\eps_\mathrm{score}(s)}{\sqrt s}\inte s}.
}
\end{theorem}

\begin{proof}
See \cref{pr:main1}.
\end{proof}

\begin{rmk}
In what follows, we do not attempt to optimize numeric constants. 
\end{rmk}

\begin{rmk}[Variance schedule]\label{rmk:sch1}
Since $h_i=t_{i+1}\beta_i\leq\beta_i$ for all $i\in\indices$, we can take $\eta=\max_{i\in\indices}\beta_i$ in \cref{thm:main1}. 
Note that we need $t_0$ to be small. 
Since $\log(1-x)\geq-4x$ for any $x\in[0,3/4]$, we have $\log t_0\geq-4\sum_{i=0}^{N-1}\beta_i
\geq-4N\max_{i\in\indices}\beta_i$ under the assumptions of \cref{thm:main1}. 
Hence, to ensure $t_0=O(N^{-c})$ for some $c>0$, $\max_{i\in\indices}\beta_i$ must be at least of order $N^{-1}\log N$. 
A simple way to remove this $\log N$ factor while keeping $t_0=O(N^{-1})$ is to set $t_i=(i+1)/(N+1)$ for $i=0,1,\dots,N$ and define $\beta_i$ by \eqref{eq:beta-time}, which yields $\beta_i=1/(i+2)$ for $i\in\indices$. 
This variance schedule is used in \cite{SDWMG15} for a binomial diffusion model (see also \cite{SaLi23}). 
Alternatively, the cosine schedule in \cref{ex:cos} satisfies the assumptions of \cref{thm:main1} with $\eta=O(N^{-1})$ and $t_0=O(N^{-1})$; see \cref{prop:schedule}. 
\end{rmk}

\begin{rmk}[Score approximation error]
When $\eps_\mathrm{score}(t)=\eps(1-t)^{-1/2}$ for some constant $\eps>0$ as suggested by \cite[Theorem 3.1]{OAS23}, we have
\[
\int_{t_0}^{t_N}\frac{\eps_\mathrm{score}(s)}{\sqrt s}\inte s\leq4\eps.
\]
\end{rmk}

Let us discuss typical examples satisfying \eqref{score-lip}. 
To verify the lower bound in \eqref{score-lip}, we introduce the following definition. 
\begin{definition}[Semi-log-convexity]
Let $\beta>0$ be constant. 
A function $f:\mathbb R^d\to[0,\infty)$ is said to be \emph{$\beta$-semi-log-convex} if the function $x\mapsto e^{\frac{\beta}{2}|x|^2}f(x)$ is log-convex. 
We say that $\target$ is $\beta$-semi-log-convex if $\target$ has a $\beta$-semi-log-convex density. 
\end{definition}

We refer to \cite[p.197]{Si11} for the precise definition of log-convexity. 
In particular, if a function $f:\mathbb R^d\to[0,\infty)$ is $\beta$-semi-log-convex for some $\beta>0$, then $f(x)>0$ for all $x\in\mathbb R^d$ by \cite[Proposition 13.5(iv)]{Si11}, unless $f$ is identically zero. 
Also, if $f$ is differentiable, the semi-log-convexity of $f$ is equivalent to the following one-sided gradient Lipschitz property of $\log f$.
\begin{lemma}
Let $\beta>0$ be constant and $f:\mathbb R^d\to[0,\infty)$ a differentiable function that is not identically zero. 
Then, $f$ is $\beta$-semi-log-convex if and only if $f(x)>0$ for all $x\in\mathbb R^d$ and 
\ben{\label{eq:one-sided}
\langle\nabla\log f(x)-\nabla\log f(y),\,x-y\rangle\geq \beta|x-y|^2\quad\text{for all }x,y\in \mathbb R^d. 
}
\end{lemma}

\begin{proof}
The claim is a direct consequence of the well-known characterization of convex functions; see, e.g., \cite[Theorem 3.8.3]{NiPe18}.
\end{proof}

The next result provides a lower bound for $\nabla^2\log f_t$ under the semi-log-convexity. 
\begin{lemma}[\cite{MiSh23}, Lemma 13]\label{lem:lb}
If $\target$ is $\beta$-semi-log-convex for some $\beta>0$, then for all $0<t<1$,
\[
\nabla^2\log f_t(x)\succeq\frac{t(1-\beta)}{(1-t)(\beta-1)+1}I_d.
\]
\end{lemma}

\begin{proof}
To apply \cite[Lemma 13]{MiSh23}, we rewrite $f_t$ in terms of the OU semigroup. 
By assumption, $\target$ has a density $f$ with respect to $N(0,I_d)$. 
A straightforward computation shows
\ben{\label{eq:ou}
f_{e^{-2s}}(x)
=\int_{\mathbb R^d}f(e^{-s}x+\sqrt{1-e^{-2s}}z)\phi_d(z)\inte z
}
for all $s>0$ and $x\in\mathbb R^d$ (this also follows from the fact that the OU semigroup is symmetric with respect to $N(0,I_d)$). 
Given this expression, we infer the desired result from \cite[Lemma 13]{MiSh23}.
\end{proof}

Next we turn to the upper bound for $\nabla^2\log f_t$. 
Following \cite{CLP25,GSO25,KITL25}, we introduce the next definition:
\begin{definition}[Weak convexity]
For a $C^1$ function $g:\mathbb R^d\to\mathbb R$, its \emph{weak convexity profile} is defined by
\[
\kappa_g(r):=\inf\cbra{\frac{\langle\nabla g(x)-\nabla g(y),\,x-y\rangle}{|x-y|^2}:x,y\in\mathbb R^d,\,|x-y|=r},\quad r>0.
\]
For constants $\alpha>0$ and $M\geq0$, we say that $g$ is \emph{$(\alpha,M)$-weakly convex} if
\[
\kappa_g(r)\geq \alpha-\frac{2\sqrt M}{r}\tanh\bra{\frac{\sqrt Mr}{2}}\quad\text{for all }r>0.
\]
We say that $g$ is \emph{$(\alpha,M)$-weakly concave} if $-g$ is $(\alpha,M)$-weakly convex. 
Also, we say that $\target$ is \emph{$(\alpha,M)$-weakly log-concave} if $\target$ has a  density $\pi$ such that $\log\pi$ is $(\alpha,M)$-weakly concave. 
\end{definition}
%
Since $(\alpha,0)$-weak concavity is equivalent to $\alpha$-strong concavity, all strongly log-concave distributions are weakly log-concave. 
However, a general log-concave distribution need not be weakly log-concave. 
Indeed, we will later see that weak log-concavity implies sub-Gaussianity; see \cref{prop:lip-trans}. 
We note that this fact was already pointed out in \cite{St26}. 

The next result provides an upper bound for $\nabla^2\log f_t$ under weak log-concavity. 
\begin{lemma}[\cite{CLP25}, Lemma 5.9]\label{lem:ub-wlc}
Suppose that $\target$ is $(\alpha,M)$-weakly log-concave for some constants $\alpha>0$ and $M\geq0$. Then, for any $0<t<1$ and $x\in\mathbb R^d$,
\[
\nabla^2\log f_t(x)\preceq-\frac{t}{\alpha+(1-\alpha)t}\bra{\alpha-1-\frac{2M}{\alpha+(1-\alpha)t}}I_d.
\]
\end{lemma}

\begin{proof}
Let $f$ be the density of $\target$ with respect to $N(0,I_d)$ and observe that $\log f$ is $(\alpha-1,M)$-weakly concave by assumption. Then, the asserted claim follows from \eqref{eq:ou} and \cite[Lemma 5.9]{CLP25}.
\end{proof}

Combining Lemmas \ref{lem:lb} and \ref{lem:ub-wlc} with \cref{thm:main1}, we obtain the following bound.
\begin{corollary}[Weakly log-concave and semi-log-convex distributions]\label{coro1}
Assume \ref{ass:score}, $\delta=0$ and $\max_{i\in\indices}\beta_i\leq3/4$. 
Suppose that $\target$ is $\beta$-semi-log-convex for some $\beta>0$ and $(\alpha,M)$-weakly log-concave for some $\alpha>0$ and $M\geq0$.
Then,
\ben{\label{eq:wlc-slc}
\wass_2(\ddpm_N,\target)
\leq e^{\lambda}\cbra{t_0\bra{|\mu-\mean|+\lambda\sqrt{dt_0}}+\lambda\sqrt{d}\eta+2\int_{t_0}^{t_N}\frac{\eps_\mathrm{score}(s)}{\sqrt s}\inte s},
}
where
\be{
\lambda:=\max\cbra{\beta-1,\,\frac{(1-\alpha)\vee0}{\alpha}+\frac{2M}{\alpha^2\wedge1}}.
}
\end{corollary}

A prominent example of weakly log-concave and semi-log-convex distributions is a finite mixture of multivariate normal distributions; see \cite[Proposition 4.1]{GSO25}. 

When $\alpha=\beta=1$, $M=0$, $\mu=\mean$ and $\eps_\mathrm{score}\equiv0$, the right-hand side of \eqref{eq:wlc-slc} vanishes. This is expected because in this case $\target=N(\mean,I_d)$ and $\scored_i(\ddpm_i)=\nabla\log\pi_{t_i}(\ddpm_i)$ for all $i\in\indices$, which yields $\ddpm_i\sim\target$ for all $i=0,1,\dots,N$. 
More generally, when $\target$ is Gaussian with a spherical covariance matrix, we have the following lower bound for $\wass_2(\ddpm_N,\target)$.

\begin{proposition}\label{prop:lb}
Suppose that $\target=N(\mu,\sigma^2I_d)$ for some $\mu\in\mathbb R^d$ and $\sigma>0$. 
Suppose also that $\delta=0$, $\eps_\mathrm{score}\equiv0$, and $\min_{i\in\indices}\beta_i\geq\eta$ for some $0<\eta<1$. 
Then, 
\ben{\label{eq:lb}
\wass_2(\ddpm_N,\target)
\geq\frac{\sigma^2}{\sigma^2\vee1}\min\cbra{t_0|\mean-\mu|,\,\frac{\sigma^2|\sigma^2-1|}{2(\sigma^3\vee1)}\bra{\sqrt dt_0^2+\frac{1-t_0^2}{2(\sigma^2\vee1)}\sqrt d\eta}}.
}
\end{proposition}

\begin{proof}
See \cref{pr-prop:lb}.
\end{proof}

We note that Proposition 6 in \cite{GNZ25} gives a similar lower bound when $\mu=\mean=0$ for reverse diffusion samplers. 
\cref{prop:lb} shows that the appearance of the terms $t_0|\mu-\mean|$ and $\sqrt d\eta$ on the right-hand side of \eqref{eq:wlc-slc} is optimal. 
On the other hand, the term $\sqrt dt_0^{3/2}$ in \eqref{eq:wlc-slc} may be suboptimal, i.e., the lower bound \eqref{eq:lb} suggests that it might be replaced by $\sqrt dt_0^2$. 
We leave the optimality of this term for future work. 
We note that the term $\sqrt dt_0^{3/2}$ appears essentially when bounding $\wass_2(t_0(\xi-\mu)+\sqrt{t_0(1-t_0)}Z,N(0,t_0I_d))$, where $\xi\sim\target$ and $Z\sim N(0,I_d)$ are independent; see \eqref{eq:initialize}. 
For this purpose, the $O(\sqrt dt_0^{3/2})$ bound is sharp in view of \cite[Proposition 7]{BeBa25} (see also \cite[Theorem 2.1]{ChNW22}).

Below we focus mainly on the dependence on the ``maximal step size'' $\eta$, as it determines the iteration complexity. 
We have already seen that the optimal dependence is of order $\sqrt d\eta$. 
To the best of the author's knowledge, there are three studies providing 2-Wasserstein bounds of order $\sqrt d\eta$ for the DDPM: \citet[Theorems 2 and 3]{AVD25}, \citet[Theorem 3.4]{WaWa25} and \citet[Corollary 12]{St26}. 
\citet{WaWa25} assume that $\target$ is a normal distribution perturbed by a $C^2$ function with bounded derivatives. In this case, $\target$ is weakly log-concave and semi-log-convex, so \cref{coro1} is applicable.\footnote{Note that \citet{WaWa25} assume \ref{ass:score-p} instead of \ref{ass:score}, so their result is not directly covered by \cref{coro1}. See also \cref{thm:main3}.}
On the other hand, \citet{AVD25} have shown that the lower bound restriction in \eqref{score-lip} can be dropped if the variance schedule is appropriately designed. 
The key observation of \cite{AVD25} is that one \emph{always} has 
\ben{\label{eq:hess-lb}
\nabla^2\log f_t(x)\succeq-\frac{t}{1-t}I_d\quad\text{for all }0<t<1\text{ and }x\in\mathbb R^d
}
(cf.~\cite[Lemma 1.3]{ElLe18} or \eqref{eq:v-deriv}). 
Building on this idea, we obtain the following result.
\begin{theorem}\label{thm:main2}
Assume \ref{ass:score} and $\max_{i\in\indices}\beta_i\leq3/4$. 
Suppose that there exists a non-decreasing function $L:(0,1)\to[0,\infty)$ such that 
\ben{\label{score-lip-one}
\nabla^2\log f_t(x)\preceq tL(t)I_d
\quad\text{for all }0<t<1\text{ and }x\in\mathbb R^d. 
}
Suppose also that there exists a constant $\eta>0$ such that
\ben{\label{eq:sampling}
h_i\leq\eta\sqrt{1-t_{i+1}}\quad\text{for all }i\in\indices.
}
Further, suppose $t_0\vee\delta\leq1/2$. 
Then,
\bm{
\wass_2(\ddpm_N,\target)
\leq \sqrt 2e^{\int_{t_0}^{t_N}L(s)\inte s}\left\{t_0|\mu-\mean|+\sqrt dt_0\bra{2\sqrt{t_0}\vee\sqrt{\int_{0}^{t_{0}}L(s)^2\inte s}}\right.\\
\left.+\sqrt{d\log(1/\delta)+\trace(\Sigma)\vee d}\,\eta+2\int_{t_0}^{t_N}\frac{\eps_\mathrm{score}(s)}{\sqrt s}\inte s\right\}
+\sqrt{\delta^2\trace(\Sigma)+d\delta}.
}
\end{theorem}

\begin{proof}
See \cref{pr:main2}.
\end{proof}

\begin{rmk}[Variance schedule]\label{rmk:sch2}
For the analysis of TV and KL bounds, it is commonly assumed that there exists a constant $0<\eta<1$ such that 
\ben{\label{eq:sampling2}
h_i\leq\eta\min\{t_i,1-t_{i+1}\}\quad\text{for all }i\in\indices.
}
See \cite[Eq.(10)]{JZL25} for example.  
\eqref{eq:sampling} is weaker than this condition. 
In particular, the geometric schedule \eqref{eq:geo}, which is essentially the variance schedule adopted in \cite{AVD25}, satisfies \eqref{eq:sampling2} with $\eta=O(N^{-1}\log N)$ by \cite[Eq.(39b)]{LWCC24}; hence it also satisfies \eqref{eq:sampling} with the same $\eta$. 
Moreover, the cosine schedule satisfies \eqref{eq:sampling} with $\eta=O(N^{-1})$ when $s\geq(2N)^{-1}$, as shown in the following proposition. 
\begin{proposition}\label{prop:schedule}
Let $(t_i)_{i=0}^{N}$ be given as in \cref{ex:cos}. 
Then, the following assertions hold.
\begin{enumerate}[label={\normalfont(\alph*)}]

\item If $t_1\leq4t_0$, then $t_0\geq\eta_N^2/16$ and $\max_{i\in\indices}\beta_i\leq3/4$.

\item $h_i\leq2\eta_N\sqrt{t_{i+1}}$ for all $i\in\indices$. 

\item If $s\geq(2N)^{-1}$, then $h_i\leq4\eta_N\sqrt{1-t_{i+1}}$ for all $i\in\indices$. 

\end{enumerate}
\end{proposition}

\begin{proof}
See \cref{pr-prop:schedule}. 
\end{proof}
\end{rmk}

\cref{thm:main2} and \cref{lem:ub-wlc} immediately yield the following bound. 
\begin{corollary}[Weakly log-concave distributions]\label{coro2}
Assume \ref{ass:score}, $\max_{i\in\indices}\beta_i\leq3/4$, \eqref{eq:sampling}, and $t_0\leq1/2$. 
Suppose that $\target$ is $(\alpha,M)$-weakly log-concave for some $\alpha>0$ and $M\geq0$.
Then,
\bmn{\label{main2-applied}
\wass_2(\ddpm_N,\target)
\leq \sqrt2e^{\lambda}\left\{t_0|\mu-\mean|+(2\vee\lambda)\sqrt{d}t_0^{3/2}\right.\\
\left.+\sqrt{d\log(1/\delta)+\trace(\Sigma)\vee d}\,\eta
+2\int_{t_0}^{t_N}\frac{\eps_\mathrm{score}(s)}{\sqrt s}\inte s\right\}
+\sqrt{\delta^2\trace(\Sigma)+d\delta},
}
where 
\[
\lambda:=\frac{(1-\alpha)\vee0}{\alpha}+\frac{2M}{\alpha^2\wedge1}.
\]
\end{corollary}

The assumption \eqref{score-lip-one} covers the settings considered in \citet{AVD25}. 
Indeed, they assume 
\[
\nabla^2\log \pi_t(x)\preceq\bra{\frac{t\varphi(\sqrt{t^{-1}-1})}{(1-t)^2}-\frac{1}{1-t}}I_d
\quad\text{for all }0<t<1\text{ and }x\in\mathbb R^d, 
\]
where $\varphi:(0,\infty)\to(0,\infty)$ has one of the following forms for some constants $m>0$ and $M,b\geq0$:
\[
\varphi(\sigma)=\frac{\sigma^2}{1+m\sigma^2}+\frac{bM^2\sigma^4}{(1+M\sigma^2)^2}
\quad\text{or}\quad
\varphi(\sigma)=b\wedge\frac{\sigma^2}{(1-M\sigma^2)_+}.
\] 
One can verify that \eqref{score-lip-one} is satisfied with $L(t)\equiv m^{-1}+b(M^2\vee1)$ in the first case and $L(t)\equiv 2(b\vee1)(1+M^2)$ in the second case. 
However, we note that the bounds of \cite{AVD25} exhibit better dependence on $m,M$ and $b$. 
Finally, \citet{St26} assume that $\target$ is a log-H\"older perturbation of a strongly log-concave distribution.\footnote{Note that \citet{St26} refers to this setting as the weakly log-concave case.} 
In this case, \eqref{score-lip-one} is satisfied for some $L$ such that $\int_0^1L(s)\inte s$ is bounded by a dimension-free constant by \cite[Corollary 3]{St26}. 

As announced in \cref{sec:score}, we can replace \ref{ass:score} by its population counterpart if we impose a Lipschitz-type condition on the learned score. 
As an illustration, we give a counterpart of \cref{thm:main2}:
\begin{theorem}\label{thm:main3}
Assume \ref{ass:score-p} and $\max_{i\in\indices}\beta_i\leq3/4$. 
Suppose that there exists a non-decreasing function $L:(0,1)\to[0,\infty)$ such that for all $i\in\indices$, $\scored_i$ is Lipschitz continuous and satisfies
\ben{\label{scored-lip}
-\frac{2t_i}{1-t_i}I_d\preceq I_d+\nabla\scored_i(x)\preceq t_iL(t_i)I_d
\quad\text{a.e. }x\in\mathbb R^d.
}
Suppose also that there exists a constant $\eta>0$ satisfying \eqref{eq:sampling}. 
Further, suppose $t_0\vee\delta\leq1/2$. 
Then,
\ba{
\wass_2(\ddpm_N,\target)
&\leq \sqrt 2e^{\int_{t_0}^{t_N}L(s)\inte s}\left\{t_0\bra{|\mu-\mean|+\sqrt{\trace(\Sigma)+d}}\right.\\
&\left.\quad+\sqrt{d\log(1/\delta)+\trace(\Sigma)\vee d}\,\eta+2\int_{t_0}^{t_N}\frac{\eps_\mathrm{score}(s)}{\sqrt s}\inte s\right\}
+\sqrt{\delta^2\trace(\Sigma)+d\delta}.
}
\end{theorem}

\begin{proof}
See \cref{pr:main3}.
\end{proof}

The assumption \eqref{scored-lip} is motivated by the corresponding conditions on the population score, namely \eqref{eq:hess-lb} and \eqref{score-lip-one}. 
Note that \cref{thm:main3} does not require any regularity assumption on $\target$ beyond finite variance, although it would be unreasonable to impose \eqref{scored-lip} unless $\target$ satisfies a condition such as \eqref{score-lip-one}. 


Theorems \ref{thm:main1}--\ref{thm:main3} become less useful when $\int_0^1L(s)\inte s=\infty$. 
To the best of the author's knowledge, all existing bounds of order $\sqrt d\eta$ assume conditions implying either \eqref{score-lip-one} or \eqref{scored-lip} with $\int_0^1L(s)\inte s<\infty$. 
For \eqref{score-lip-one}, this condition forces $\target$ to be sub-Gaussian. 
In fact, we have the following stronger result. 

\begin{lemma}\label{prop:lip-trans}
Suppose that there exists a measurable function $L:(0,1)\to[0,\infty)$ satisfying \eqref{score-lip-one} and $\bar L:=\int_0^1L(s)\inte s<\infty$. 
Then, there exists an $e^{\bar L}$-Lipschitz function $T:\mathbb R^d\to\mathbb R^d$ such that $\target$ is the push-forward of $N(0,I_d)$ by $T$. 
In particular, $\int_{\mathbb R^d}e^{a|x|^2}\target(\mathrm{d}x)<\infty$ for any $0<a<2^{-1}e^{-2\bar L}$. 
\end{lemma}

\begin{proof}
Noting $\int_0^1L(s)\inte s=\int_0^\infty e^{-2\tau}L(e^{-2\tau})\inte \tau$ and using \eqref{eq:ou}, we infer the desired result from \cite[Lemma 11]{MiSh23} or \cite[Lemma 3.2]{Ne22}.
\end{proof}

Another important consequence of this result is a logarithmic Sobolev inequality for $\target$. 
Recall that a probability distribution $P$ on $\mathbb R^d$ is said to satisfy a \emph{logarithmic Sobolev inequality} $LS(C)$ with constant $C>0$ if for any smooth function $f:\mathbb R^d\to\mathbb R$ with $f\in L^2(P)$,
\[
\int_{\mathbb R^d}f^2\log f^2\inte P-\int_{\mathbb R^d}f^2\inte P\log\bra{\int_{\mathbb R^d}f^2\inte P}\leq2C\int_{\mathbb R^d}|\nabla f|^2\inte P.
\]
We refer to \cite[Chapter 5]{BGL14} for a detailed account of logarithmic Sobolev inequalities. 
Under the assumptions of \cref{prop:lip-trans}, $\target$ satisfies $LS(e^{2\bar L})$ by \cite[Lemma 16]{VeWi23}. 
Moreover, the law of the DDPM sample also satisfies a logarithmic Sobolev inequality if the learned score satisfies the Lipschitz condition \eqref{scored-lip}:
\begin{theorem}\label{lem:LS}
Suppose that there exists a non-decreasing function $L:(0,1)\to[0,\infty)$ such that for all $i\in\indices$, $\scored_i$ is Lipschitz continuous and satisfies \eqref{scored-lip}. 
Then, the law of $\ddpm_N$ satisfies $LS(2e^{2\int_{t_0}^{t_N}L(s)\inte s})$.
\end{theorem}
The proof of this theorem is analogous to \cite[Theorem 8]{VeWi23} and is given in \cref{pr-lem:LS}. 
This result suggests that one can derive an error bound in the 2-Wasserstein distance from a corresponding KL bound. 
Recall that a probability distribution $P$ on $\mathbb R^d$ is said to satisfy a \emph{quadratic transportation cost inequality} $T_2(C)$ with constant $C>0$ if for any probability distribution $Q$ on $\mathbb R^d$,
\[
\wass_2(P,Q)\leq\sqrt{2C\kl(Q\mid P)}.
\]
Combining \cref{lem:LS} with the celebrated Otto--Villani theorem (see \cite[Theorem 9.6.1]{BGL14}), we obtain the following corollary.
\begin{corollary}\label{coro:t2}
Under the assumptions of \cref{lem:LS}, the law of $\ddpm_N$ satisfies $T_2(2e^{2\int_{t_0}^{t_N}L(s)\inte s})$.
\end{corollary}
Notably, up to a log factor, an error bound of order $\sqrt d\eta$ in the 2-Wasserstein distance can be derived from the recent sharp KL bound of \citet{JZL25} under a one-sided Lipschitz condition on the score function and the variance schedule condition \eqref{eq:sampling2}. 
Specifically, we have the following result.
\begin{proposition}[cf.~\cite{JZL25}, Theorem 1]\label{prop:kl}
Assume that there exists a non-decreasing function $\eps_\mathrm{score}:(0,1)\to[0,\infty)$ satisfying \eqref{score-error-p}. 
Suppose that there exists a constant $L_0\geq1$ such that
\ben{\label{ass:kl}
\nabla^2\log \pi_t(x)\preceq\frac{L_0}{1-t}I_d\quad\text{for all }0<t<1\text{ and }x\in\mathbb R^d.
}
Suppose also that \eqref{eq:sampling2} is satisfied for some constant $0<\eta<1/6$. 
Further, assume $t_0\vee\delta\leq1/2$. 
Then,
\ben{\label{eq:kl-bound}
\kl(\slepian_{t_N}\target\mid P^{\ddpm_N})
\leq \frac{t_0}{2}\bra{|\mu-\mean|^2+\trace(\Sigma)}+2t_0^2d
+CL_0^2d\eta^2\log\frac{1}{t_0\delta}+2\int_{t_0}^{t_N}\frac{\eps_\mathrm{score}(s)^2}{s}\inte s,
}
where $C$ is a universal constant.  
\end{proposition}

%

In fact, \citet{JZL25} obtained a bound similar to \eqref{eq:kl-bound} under a weaker version of \eqref{ass:kl}. We give a self-contained proof in \cref{pr-prop:kl}. 
Under the assumptions of \cref{thm:main3}, we have by \cref{coro:t2} and \cref{lem:es}
\[
\wass_2(\ddpm_N,\target)
\leq2e^{\int_{t_0}^{t_N}L(s)\inte s}\sqrt{\kl(\slepian_{t_N}\target\mid P^{\ddpm_N})}+\sqrt{\delta^2\trace(\Sigma)+d\delta}.
\]
Hence, if we additionally assume \eqref{eq:sampling2} and \eqref{ass:kl}, we obtain an error bound of order $\sqrt d\eta$ for $\wass_2(\ddpm_N,\target)$ up to a log factor. 
We note that a similar proof strategy also applies under the assumptions of \cref{thm:main2}. 
Indeed, let $(\ddpmy_i)_{i=0}^N$ be defined by \eqref{eq:ddpm} with $\scored_i$ replaced by $\nabla\log\pi_{t_i}$. 
Under the assumptions of \cref{thm:main2}, $\ddpmy_N$ satisfies $T_2(2e^{2\int_{t_0}^{t_N}L(s)\inte s})$ by \cref{coro:t2}. 
Moreover, in this case, we can prove the following bound (see \cref{lem:score-error} and \eqref{eq:p-quad}):
\[
\wass_2(\ddpm_N,\ddpmy_N)\leq2\sqrt2\int_{t_0}^{t_N}\frac{\eps_\mathrm{score}(s)}{\sqrt s}\inte s.
\]
Combining these facts with \cref{lem:es}, we obtain
\ba{
\wass_2(\ddpm_N,\target)
&\leq\wass_2(\ddpm_N,\ddpmy_N)+\wass_2(\ddpmy_N,\slepian_{t_N}\target)+\wass_2(\slepian_{t_N}\target,\target)\\
&\leq2\sqrt2\int_{t_0}^{t_N}\frac{\eps_\mathrm{score}(s)}{\sqrt s}\inte s+2e^{\int_{t_0}^{t_N}L(s)\inte s}\sqrt{\kl(\slepian_{t_N}\target\mid P^{\ddpmy_N})}+\sqrt{\delta^2\trace(\Sigma)+d\delta}.
}
The second term on the last line can be bounded by \cref{prop:kl} if \eqref{eq:sampling2} and \eqref{ass:kl} are additionally imposed. 
Note that $\eqref{ass:kl}$ is typically weaker than \eqref{score-lip-one}.


\begin{rmk}
When $\target$ is supported on a Euclidean ball $B$ of radius $R>0$, one can convert a TV error bound for $\ddpm_N$ into a Wasserstein bound for $\ddpm_N$ projected onto $B$; see \cite[Corollary 6]{chen2023sampling}, \cite[Theorem 2.1]{LLT23} and \cite[Appendix D]{WaWa25}. 
However, as argued in \cite{WaWa25}, it is unclear how to control the projection error in this setting, and thus these results are not directly comparable with Wasserstein bounds for $\ddpm_N$ itself. 
By contrast, our argument based on \cref{coro:t2} applies directly to $\ddpm_N$. 
\end{rmk}

Given these observations, it is natural to ask whether one can obtain a bound of order $\sqrt d\eta$ even when the target distribution does not satisfy a quadratic transportation cost inequality. 
In the next two results, we show that this is possible for log-concave distributions. 
The first result applies when $\mean$ is close to $\mu$.  
\begin{theorem}\label{prop:lc}
Assume \ref{ass:score} and $\max_{i\in\indices}\beta_i\leq3/4$. 
Suppose that $\target$ is log-concave. 
Suppose also that there exists a constant $\eta>0$ such that
\ben{\label{eq:sampling3}
h_i\leq\min\cbra{\eta\sqrt{t_{i+1}},\,\eta\sqrt{1-t_{i+1}}}\text{ for all }i=0,1,\dots,N-1
\quad\text{and}\quad
t_0\geq2^{-8}\eta^2.
}
Further, assume $t_0\vee\delta\leq1/2$. 
Then, there exists a universal constant $C$ such that
\ba{
\wass_2(\ddpm_N,\target)
&\leq C\cbra{|\mu-\mean|+\Lambda \sqrt{dt_0}
+\bra{\Lambda\sqrt{d\log(1/t_0)}+\sqrt{d\log(1/\delta)}}\eta
+\int_{t_0}^{t_N}\frac{\eps_\mathrm{score}(s)}{s^{3/2}}\inte s}\\
&\quad+\sqrt{\delta^2\trace(\Sigma)+d\delta},
}
where $\Lambda:=1\vee\|\Sigma\|_{op}$.
\end{theorem}

\begin{proof}
See \cref{pr-prop:lc}. 
\end{proof}

\begin{rmk}[Variance schedule]
The condition \eqref{eq:sampling3} is still weaker than \eqref{eq:sampling2} as long as $t_0\geq2^{-8}\eta^2$. 
In particular, by \cref{prop:schedule}, it is satisfied with $\eta=O(N^{-1})$ by the cosine schedule if $t_1\leq4t_0$ and $s\geq(2N)^{-1}$. 
\end{rmk}

\begin{rmk}[Score approximation error]
When $\eps_\mathrm{score}(t)=\eps(1-t)^{-1/2}$ for some constant $\eps>0$ as suggested by \cite[Theorem 3.1]{OAS23}, we have
\[
\int_{t_0}^{t_N}\frac{\eps_\mathrm{score}(s)}{s^{3/2}}\inte s\leq\bra{\frac{2\sqrt 2}{\sqrt{t_0}}+4}\eps.
\]
\end{rmk}

\cref{prop:lc} is meaningful only when $|\mu-\mean|$ is small. 
With additional work, we can derive a bound that remains useful even when $|\mu-\mean|$ is not small.

\begin{theorem}\label{thm:lc}
Under the assumptions of \cref{prop:lc}, assume also that $t_0\leq(\Lambda d^2)^{-1}$ and $\sqrt{\Lambda d\log(1/t_0)}\eta\leq1$. 
Then, there exist positive universal constants $C'$ and $c$ such that if $t_0\leq c(\Lambda\log^4(d+1))^{-1}$, then
\ba{
\wass_2(\ddpm_N,\target)
&\leq C'\left(
\sqrt{\Lambda t_0}|\mu-\mean|+\Lambda^{-1/4}t_0^{1/4}|\mu-\mean|^2
+\Lambda^{3/2}\sqrt dt_0\right)\\
&\quad+C'\bra{\Lambda\sqrt{d\log(1/t_0)}+\sqrt{d\log(1/\delta)}}\eta
+4\int_{t_0}^{t_N}\frac{\eps_\mathrm{score}(s)}{s^{3/2}}\inte s
+\sqrt{\delta^2\trace(\Sigma)+d\delta}.
}
\end{theorem}

\begin{proof}
See \cref{pr-thm:lc}. 
\end{proof}

\section{Proofs of the main results}\label{sec:proof}

We introduce some additional notation. 
For a function $F:(0,1)\times\mathbb R^d\to\mathbb R^k$ such that the partial derivative $\frac{\partial F}{\partial x_j}(t,x)\in\mathbb R^k$ exists at $(t,x)\in(0,1)\times\mathbb R^d$ for every $j=1,\dots,d$, we write
\[
\nabla F(t,x):=\begin{pmatrix}
\displaystyle\frac{\partial F}{\partial x_1}(t,x) & \dots &\displaystyle\frac{\partial F}{\partial x_d}(t,x)
\end{pmatrix}
\in\mathbb R^k\otimes\mathbb R^d.
\] 
For two real numbers $x$ and $y$, we write $x\lesssim y$ if there exists a universal constant $C>0$ such that $x\leq Cy$. 

\subsection{DDPM sampler as a discretized F\"ollmer process}\label{sec:follmer}

As already mentioned in the introduction, we view the DDPM sampler as a discretization of the F\"ollmer process rather than the conventional reverse OU process. 
Following \cite{ELS20}, we call a continuous stochastic process $X=(X_t)_{t\in[0,1]}$ in $\mathbb R^d$ the \emph{F\"ollmer process} associated with $\target$ if $X_1\sim\target$ and the conditional law of $X$ given $X_1$ is the $d$-dimensional Brownian bridge from 0 to $X_1$ on $[0,1]$. 
It is known that $X$ solves the following SDE:
\ben{\label{eq:sde}
\mathrm{d}X_t=v(t,X_t)\inte t+\mathrm{d}W_t,
}
where 
\ben{\label{eq:v-def}
v(t,x)=\frac{1}{\sqrt t}\nabla\log f_{t}\bra{\frac{x}{\sqrt t}}.
}
In fact, as shown in \cite[Theorem 2.1]{Ko26ddpm}, $X$ is always a weak solution to this SDE without any assumption on $\target$. 
To be precise, we have $\int_0^1|v(t,X_t)|\inte t<\infty$ a.s., and the process
\[
W_t:=X_t-\int_0^tv(s,X_s)\inte s,\quad t\in[0,1],
\]
is a standard $\mathbf F$-Wiener process in $\mathbb R^d$, where $\mathbf F=(\mcl F_t)_{t\in [0,1]}$ is the usual augmentation of the filtration generated by $X$ (see \cite[Definition 2.7.2]{KaSh98} or \cite[p.45]{ReYo99}). 

Let us consider the Euler--Maruyama discretization of the SDE \eqref{eq:sde} over the time grid $t_0<t_1<\cdots<t_{N-1}<t_N$. 
It is given by
\ben{\label{eq:em}
\begin{cases}
\eul_{t_0}:=t_0\mean+W_{t_0},\\
\eul_{t_{i+1}}:=\eul_{t_i}+h_iv(t_i,\eul_{t_i})+W_{t_{i+1}}-W_{t_i}\quad(i=0,1,\dots,N-1).
\end{cases}
}
Here, we use $\mean$ as a proxy for the generally undefined quantity $v(0,0)$. 
Using \eqref{eq:v-def}, we can rewrite the second equation in \eqref{eq:em} as
\[
\eul_{t_{i+1}}=\bra{1+\frac{h_i}{t_i}}\eul_{t_i}+\frac{h_i}{\sqrt{t_i}}\nabla\log\pi_{t_i}\bra{\frac{\eul_{t_i}}{\sqrt{t_i}}}+W_{t_{i+1}}-W_{t_i}.
\]
Replacing the score $\nabla\log\pi_{t_i}$ by its proxy $\scored_i$, we define random vectors $(\euler_{t_i})_{i=0}^N$ recursively by
\be{
\begin{cases}
\euler_{t_0}:=t_0\mean+W_{t_0},\\
\displaystyle\euler_{t_{i+1}}:=\bra{1+\frac{h_i}{t_i}}\euler_{t_i}+\frac{h_i}{\sqrt{t_i}}\scored_{i}\bra{\frac{\euler_{t_i}}{\sqrt{t_i}}}+W_{t_{i+1}}-W_{t_i}\quad(i=0,1,\dots,N-1).
\end{cases}
}
As pointed out in \cite{Ko26ddpm}, the rescaled sequence $(t_i^{-1/2}\euler_{t_i})_{i=0}^N$  satisfies the recursive relation \eqref{eq:ddpm}; hence it has the same law as $(\ddpm_i)_{i=0}^N$.

A main technical advantage of working with the F\"ollmer process is that its drift process
\[
v_t:=v(t,X_t),\quad t\in(0,1),
\]
is a martingale with respect to $\mathbf F$; see \cite[Fact 2.2]{ElLe18} or \cite[Proposition 2.8]{Ko26ddpm}. 
Moreover, we have the following explicit martingale representation formula. 
\begin{lemma}\label{lem:v-rep}
For any $0<t<1$,
\ben{\label{eq:v-rep}
v_t=\mu+\int_0^t\nabla v(s,X_s)\inte W_s.
}
\end{lemma}

This result is known in the literature at least when $\target$ has density; see, e.g., \cite[Proposition 11]{ELS20}. 
For completeness, we give a proof below. 
For the proofs of \cref{lem:v-rep} and our main theorems, it is useful to introduce another representation of $v(t,x)$ in terms of conditional moments of $X_1$ given $X_t=x$. 
For $t\in(0,1)$ and $x\in\mathbb R^d$, define a probability distribution $\target_{t,x}$ on $\mathbb R^d$ by
\[
\target_{t,x}(\mathrm{d}y)=\frac{e^{\frac{|y|^2}{2}-\frac{|y-x|^2}{2(1-t)}}}{\int_{\mathbb R^d}e^{\frac{|z|^2}{2}-\frac{|z-x|^2}{2(1-t)}}\target(\mathrm{d}z)}\target(\mathrm{d}y).
\]
It is known that $\target_{t,x}$ is the (regular) conditional law of $X_1$ given $X_t=x$; see \cite[Eq.(21)]{ElLe18} or \cite[Proposition 2.5]{Ko26ddpm} for a proof. 
Also, since the function $y\mapsto|y|^pe^{\frac{|y|^2}{2}-\frac{|y-x|^2}{2(1-t)}}$ is bounded for all $p>0$, $\target_{t,x}$ has all moments. 
We particularly set
\[
m(t,x):=\int_{\mathbb R^d}y\target_{t,x}(\mathrm{d}y)
,\qquad
V(t,x)=\int_{\mathbb R^d}\bra{y-m(t,x)}^{\otimes2}\target_{t,x}(\mathrm{d}y).
\]
A straightforward computation shows
\ben{\label{v-formula}
v(t,x)=\frac{m(t,x)-x}{1-t}\quad\text{for all }0<t<1\text{ and }x\in\mathbb R^d.
}
The following lemma can also be shown by an elementary calculation; see also \cite[Lemma 2.1]{KlPu23}. 
\begin{lemma}\label{sl-deriv}
Let $h:\mathbb R^d\to\mathbb R$ be a measurable function of polynomial growth. 
For $0<t<1$, define a function $Q_t^*h:\mathbb R^d\to\mathbb R$ by  $Q_t^*h(x)=\int_{\mathbb R^d}h(y)\target_{t,x}(\mathrm{d}y)$ for $x\in\mathbb R^d$. 
Then, $Q^*_th$ is differentiable and 
\[
\nabla Q^*_th(x)=\frac{1}{1-t}\int_{\mathbb R^d}h(y)\bra{y-m(t,x)}\target_{t,x}(dy)
\quad\text{for all }x\in\mathbb R^d.
\]
\end{lemma}
Applying \cref{sl-deriv} to $h(y)=y$, we obtain the following corollary. 
\begin{corollary}
For all $0<t<1$ and $x\in\mathbb R^d$,
\ben{\label{eq:m-deriv}
\nabla m(t,x)=\frac{V(t,x)}{1-t}
}
and
\ben{\label{eq:v-deriv}
\nabla v(t,x)=\frac{V(t,x)}{(1-t)^2}-\frac{I_d}{1-t}.
}
\end{corollary}

For $0<t<1$, set
\[
m_t:=m(t,X_t)=\E[X_1\mid X_t],\qquad
V_t:=V(t,X_t)=\Cov[X_1\mid X_t],
\]
where the equalities follow from the fact that $\target_{t,x}$ is the conditional law of $X_1$ given $X_t=x$. 
Since $X$ is a Markov process by \cite[Proposition 2.2]{Ko26ddpm}, we have $m_t=\E[X_1\mid\mcl F_t]$ and $V_t=\Cov[X_1\mid \mcl F_t]$. 
Therefore, by \cite[Corollary 2.7.8]{KaSh98} and \cite[Corollary II.2.4]{ReYo99},
\ben{\label{zero-lim}
m_t\to\mu\quad\text{and}\quad V_t\to\Sigma\quad\text{a.s. }~(t\downarrow0).
}

We are now ready to prove \cref{lem:v-rep}.
\begin{proof}[\bf\upshape Proof of \cref{lem:v-rep}]
Fix $\eps\in(0,t)$ arbitrarily. By It\^o's formula, for any $r\in[\eps,t]$,
\ba{
v_r=v_\eps+\int_\eps^{r}\frac{\partial v}{\partial s}(s,X_s)\inte s
+\sum_{i=1}^d\int_\eps^{r}\frac{\partial v}{\partial x_i}(s,X_s)\inte X^i_s
+\frac{1}{2}\sum_{i=1}^d\int_\eps^{r}\frac{\partial^2v}{\partial x_i^2}(s,X_s)\inte s. 
}
Since $(v_r)_{r\in[\eps,t]}$ is a martingale, Proposition IV.1.2 in \cite{ReYo99} yields
\ben{\label{eq:v-ito}
v_t=v_\eps+\sum_{i=1}^d\int_\eps^{t}\frac{\partial v}{\partial x_i}(s,X_s) \inte W^i_s
=v_\eps+\int_\eps^t\nabla v(s,X_s)\inte W_s.
}
Since $\nabla v(s,X_s)\to\Sigma-I_d$ a.s. as $s\downarrow0$ by \eqref{eq:v-deriv} and \eqref{zero-lim}, we have $\sup_{0<s\leq t}|\nabla v(s,X_s)|<\infty$ a.s. 
Therefore, the It\^o integral $\int_0^t\nabla v(s,X_s)\inte W_s$ is well-defined and $\int_\eps^t\nabla v(s,X_s)\inte W_s\to\int_0^t\nabla v(s,X_s)\inte W_s$ a.s. as $\eps\downarrow0$ by continuity. 
Moreover, $v_\eps\to\mu$ a.s. as $\eps\downarrow0$ by \eqref{v-formula} and \eqref{zero-lim}. 
Hence, we obtain the desired result by letting $\eps\downarrow0$ in \eqref{eq:v-ito}. 
\end{proof} 

We conclude this subsection with a useful formula for the second moment of $v_t$.
\begin{lemma}\label{lem:v2}
For every $0<t<1$,
\be{
\E[|v_t|^2]=\frac{d-\E[\trace(\Gamma_t)]}{1-t}
+\trace(\Sigma)-d,
}
where
\be{
\Gamma_t:=\frac{\Cov[X_1\mid X_t]}{1-t}.
}
\end{lemma}

\begin{proof}
See \cite[Lemma 11]{EMZ20} or \cite[Proposition 2.6]{Ko26ddpm}.
\end{proof}

\subsection{Abstract error bounds}

This subsection develops abstract error bounds used to prove our main results. 
We handle the following error sources separately: (i) initialization and discretization, (ii) score approximation, and (iii) early stopping. 
For a later application, we give an error bound for (i) in the $L^p$-norm with general $p\geq2$. 
The proof is based on two lemmas. The first one is an elementary recursive inequality taken from \cite{AVD25}.
\begin{lemma}[\cite{AVD25}, Lemma 11]\label{lem:recursion}
Let $(A_i)_{i=0}^{N-1},(B_i)_{i=0}^{N-1}$ and $(C_i)_{i=0}^{N-1}$ be three sequences of real numbers such that $B_i\geq0$ and $C_i\geq0$ for every $i\in\indices$. 
If $x_i\in[0,\infty]$ $(i=0,1,\dots,N)$ satisfies
\[
x_{i+1}^2\leq(e^{A_i}x_i+B_i)^2+C_i^2\quad\text{for }i=0,1,\dots,N-1,
\]
then, for every $k=0,1,\dots,N-1$,
\[
x_{k+1}\leq e^{\bar A_{k}}x_{0}+\sum_{i=0}^{k}e^{\bar A_{k}-\bar A_i}B_i+\sqrt{\sum_{i=0}^{k}e^{2(\bar A_{k}-\bar A_i)}C_i^2},
\]
where $\bar A_i:=A_0+\cdots+A_i$.
\end{lemma}

The second lemma is an inequality for squared $L^p$-norms. 
\begin{lemma}\label{lem:lp-o}
Let $X$ and $Y$ be two random vectors in $\mathbb R^d$ such that $\E[|Y|]<\infty$. 
If $\E[Y\mid X]=0$, then for any $p\geq2$,
\ben{\label{eq:lp-o}
\|X+Y\|_p^2\leq\|X\|_p^2+(p-1)\|Y\|_p^2.
}
\end{lemma}

This inequality is a special case of \cite[Theorem 1]{RiXu16}, where the result is stated in terms of non-commutative probability. 
To avoid invoking terminology from non-commutative probability, we give a direct proof via the 2-smoothness of the $L^p$-space in \cref{sec:lp-o}. 
We note that only the case $p=2$ is necessary except in the proof of \cref{thm:lc}, and it is easy to see that \eqref{eq:lp-o} holds with equality when $p=2$. 

\begin{lemma}[Initialization and discretization errors]\label{lem:euler}
Suppose that there exists a function $L:(0,1)\to[0,\infty)$ satisfying \eqref{score-lip-one}. 
Then, for any $p\geq2$ and $k\in\{1,\dots,N\}$, 
\be{
\|X_{t_k}-\eul_{t_k}\|_p\leq e^{\lambda_{0,k}}\|X_{t_0}-\eul_{t_0}\|_p
+\sqrt{(p-1)\sum_{i=0}^{k-1}e^{2\lambda_{i+1,k}}\|\Delta_i\|_p^2},
}
where 
\[
\lambda_{i,k}:=\sum_{j=i}^{k-1}h_jL(t_j),\qquad
\Delta_i:=\int_{t_i}^{t_{i+1}}(v_r-v_{t_i})\inte r.
\] 
\end{lemma}

\begin{proof}
For every $i\in\indices$, 
\besn{\label{eq:basic-decomp}
X_{t_{i+1}}-\eul_{t_{i+1}}
&=(X_{t_{i}}-\eul_{t_{i}})+\int_{t_i}^{t_{i+1}}\bra{v(r,X_r)-v(t_i,\eul_{t_i})}\inte r\\
&=(X_{t_{i}}-\eul_{t_{i}})+h_i\bra{v(t_i,X_{t_i})-v(t_i,\eul_{t_i})}+\Delta_i\\
&=:R_i+\Delta_i.
}
Since $(v_r)_{r\in(0,1)}$ is an $\mathbf F$-martingale, we have $\E[\Delta_i\mid\mcl F_{t_i}]=0$. 
Moreover, $R_i$ is $\mcl F_{t_i}$-measurable. 
Hence, we have by \cref{lem:lp-o}
\ben{\label{basic-est}
\|X_{t_{i+1}}-\eul_{t_{i+1}}\|_p^2
\leq\|R_i\|_p^2+(p-1)\|\Delta_i\|_p^2.
}
By \eqref{eq:v-deriv} and $h_i\leq1-t_i$, for any $x\in\mathbb R^d$,
\ben{\label{eq:psd}
I_d+h_i\nabla v(t_i,x)=\bra{1-\frac{h_i}{1-t_i}}I_d+h_i\frac{V(t_i,x)}{(1-t_i)^2}\succeq 0.
}
Combining this with \eqref{score-lip-one} gives
\ben{\label{eq:lip-ex}
\|I_d+h_i\nabla v(t_i,x)\|_{op}\leq1+h_iL(t_i).
}
Hence,
\[
|R_i|\leq\bra{1+h_iL(t_i)}|X_{t_i}-\eul_{t_i}|.
\]
Thus, with $x_i:=\|X_{t_i}-\eul_{t_i}\|_p$, we have
\ben{\label{main-rec}
x_{i+1}^2\leq \bra{1+h_iL(t_i)}^2x_i^2+(p-1)\|\Delta_i\|_p^2.
}
Therefore, applying \cref{lem:recursion} with $A_i:=\log(1+h_iL(t_i)),B_i:=0$ and $C_i:=\sqrt{p-1}\|\Delta_i\|_p$, we obtain
\ben{\label{rec-result}
x_k\leq e^{\bar A_{k-1}}x_{0}+\sqrt{(p-1)\sum_{i=0}^{k-1}e^{2(\bar A_{k-1}-\bar A_i)}\|\Delta_i\|_p^2},
}
where $\bar A_i:=A_0+\cdots+A_i$. 
Since $\log x\leq x-1$ for any $x>0$, we have $\bar A_{k-1}\leq\lambda_{0,k}$ and $\bar A_{k-1}-\bar A_i\leq\lambda_{i+1,k}$ for every $i=0,1,\dots,k-1$. Hence, the desired result follows from \eqref{rec-result}.   
\end{proof}

Next we develop error bounds for score approximation and early stopping in the $L^2$-norm.  
\begin{lemma}[Score approximation error]\label{lem:score-error}
Suppose that there exists a function $L:(0,1)\to[0,\infty)$ satisfying \eqref{score-lip-one}. 
Then, under \ref{ass:score},
\[
\norm{\eul_{t_{N}}-\euler_{t_{N}}}_2\leq \sum_{i=0}^{N-1}e^{\lambda_{i+1,N}}\eps_\mathrm{score}(t_i)\frac{h_i}{\sqrt{t_i}},
\]
where $\lambda_{i,N}:=\sum_{j=i}^{N-1}h_jL(t_j)$. 
\end{lemma}

\begin{proof}
For every $i\in\indices$, we have
\ba{
&\eul_{t_{i+1}}-\euler_{t_{i+1}}\\
&=\cbra{(\eul_{t_i}-\euler_{t_i})+h_i\bra{v(t_i,\eul_{t_i})-v(t_i,\euler_{t_i})}}+\frac{h_i}{\sqrt{t_i}}\bra{\nabla\log\pi_{t_i}\bra{\frac{\euler_{t_i}}{\sqrt{t_i}}}-\scored_i\bra{\frac{\euler_{t_i}}{\sqrt{t_i}}}}\\
&=:\mathbb{I}_i+\mathbb{II}_i.
}
By \eqref{eq:lip-ex},
\[
|\mathbb I_i|\leq\bra{1+h_iL(t_i)}|\eul_{t_i}-\euler_{t_i}|.
\]
Also, by $\euler_{t_i}/\sqrt{t_i}\overset{d}{=}\ddpm_i$ and \ref{ass:score},
\ba{
\|\mathbb{II}_i\|_2&
=\frac{h_i}{\sqrt{t_i}}\|\nabla\log\pi_{t_i}(\ddpm_i)-\scored_i(\ddpm_i)\|_2
\leq\frac{h_i}{\sqrt{t_i}}\eps_\text{score}(t_i).
}
Hence, with $x_i:=\|\eul_{t_i}-\euler_{t_i}\|_2$, we have
\ba{
x_{i+1}\leq\bra{1+h_iL(t_i)}x_i+\frac{h_i}{\sqrt{t_i}}\eps_\text{score}(t_i).
}
Therefore, applying \cref{lem:recursion} with $A_i:=\log(1+h_iL(t_i)),B_i:=\eps_\text{score}(t_i)h_i/\sqrt{t_i}$ and $C_i:=0$, we obtain
\[
x_N\leq e^{\bar A_{N-1}}x_{0}+\sum_{i=0}^{N-1}e^{\bar A_{N-1}-\bar A_i}\eps_\text{score}(t_i)\frac{h_i}{\sqrt{t_i}},
\]
where $\bar A_i:=A_0+\cdots+A_i$. 
Since $\log x\leq x-1$ for any $x>0$, we have $\bar A_{N-1}-\bar A_i\leq\lambda_{i,N}$ for every $i$. Also, $x_0=0$ by construction. 
Consequently, we obtain the desired result. 
\end{proof}

\begin{lemma}[Early stopping error]\label{lem:es}
$\wass_2(\slepian_{t_N}\target,\,\target)=\wass_2(t_N^{-\frac{1}{2}}X_{t_N},\,\target)\leq\sqrt{\trace(\Sigma)\delta^2+d\delta}.$
\end{lemma}

\begin{proof}
Set $Z_t:=X_t-tX_1$ for every $t\in[0,1]$. 
By \cite[Proposition 2.1]{Ko26ddpm}, $Z_t\sim N(0,t(1-t)I_d)$ and $Z_t$ is independent of $X_1$. 
Since $X_1\sim\target$, we obtain
\ba{
\wass_2(\slepian_{t_N}\target,\,\target)^2
&=\wass_2(t_N^{-\frac{1}{2}}X_{t_N},\,\target)^2
\leq\|t_N^{-\frac{1}{2}}X_{t_N}-X_1\|_2^2\\
&=(1-\sqrt{t_N})^2\E[|X_1|^2]+t_N^{-1}\E[|Z_{t_N}|^2]
\leq\trace(\Sigma)\delta^2+d\delta.
}
Taking square roots completes the proof.
\end{proof}

Collecting the three error sources, we obtain the following bound. 
\begin{lemma}\label{lem:master}
Suppose that there exists a function $L:(0,1)\to[0,\infty)$ satisfying \eqref{score-lip-one}.  
Then
\ban{
\wass_2(\ddpm_N,\,\target)
&\leq \frac{1}{\sqrt{t_N}}\bra{e^{\lambda_{0,N}}\|X_{t_0}-\eul_{t_0}\|_2
+\sqrt{\sum_{i=0}^{N-1}e^{2\lambda_{i+1,N}}\E[|\Delta_i|^2]}
+\sum_{i=0}^{N-1}e^{\lambda_{i+1,N}}\eps_\mathrm{score}(t_i)\frac{h_i}{\sqrt{t_i}}
}\notag\\
&\quad+\sqrt{\trace(\Sigma)\delta^2+d\delta},\label{eq:master0}
}
where $\lambda_{i,N}:=\sum_{j=i}^{N-1}h_jL(t_j)$. 
Moreover, if $L$ is non-decreasing and $\max_{i\in\indices}\beta_i\leq3/4$, then
\ban{\label{eq:master}
\wass_2(\ddpm_N,\,\target)
&\leq \frac{e^{\int_{t_0}^{t_N}L(s)\inte s}}{\sqrt{t_N}}\bra{\|X_{t_0}-\eul_{t_0}\|_2+\sqrt{\sum_{i=0}^{N-1}\E[|\Delta_i|^2]}+2\int_{t_0}^{t_N}\frac{\eps_\mathrm{score}(s)}{\sqrt s}\inte s}\notag\\
&\quad+\sqrt{\trace(\Sigma)\delta^2+d\delta}.
}
\end{lemma}

\begin{proof}
Recall that $\ddpm_N$ has the same law as $t_N^{-\frac{1}{2}}\euler_{t_N}$. 
Hence,
\ban{
\wass_2(\ddpm_N,\,\target)
&\leq t_N^{-\frac{1}{2}}\|\euler_{t_N}-X_{t_N}\|_2+\wass_2(t_N^{-\frac{1}{2}}X_{t_N},\,\target)
\label{wass-es-est}\\
&\leq t_N^{-\frac{1}{2}}\bra{\|\euler_{t_N}-\eul_{t_N}\|_2+\|\eul_{t_N}-X_{t_N}\|_2}+\wass_2(t_N^{-\frac{1}{2}}X_{t_N},\,\target).\notag
}
Combining this with Lemmas \ref{lem:euler}--\ref{lem:es} gives \eqref{eq:master0}. 

Next, assume that $L$ is non-decreasing and $\max_{i\in\indices}\beta_i\leq3/4$. 
Observe that the latter implies $t_{i+1}\leq4t_i$ for every $i$. 
Hence, 
\besn{\label{eq:p-quad}
\lambda_{i,N}&\leq\int_{t_0}^{t_N}L(s)\inte s\quad\text{for all }i\in\indices,\\
\sum_{i=0}^{N-1}\eps_\mathrm{score}(t_i)\frac{h_i}{\sqrt{t_i}}&\leq2\sum_{i=0}^{N-1}h_i\frac{\eps_\mathrm{score}(t_{i})}{\sqrt{t_{i+1}}}
\leq2\int_{t_0}^{t_N}\frac{\eps_\mathrm{score}(s)}{\sqrt s}\inte s.
}
Inserting these bounds into \eqref{eq:master0} gives \eqref{eq:master}. 
\end{proof}

\subsection{Proof of Theorem \ref{thm:main1}}\label{pr:main1}

Under the assumptions of the theorem, \eqref{eq:master} holds by \cref{lem:master}. 
For all $i\in\indices$,
\ba{
\E[|\Delta_i|^2]
&\leq h_i\int_{t_i}^{t_{i+1}}\E[|v_r-v_{t_i}|^2]\inte r
=h_i\int_{t_i}^{t_{i+1}}\int_{t_i}^r\E[\|\nabla v(s,X_s)\|_F^2]\inte s\inte r
\leq dh_i^2\int_{t_i}^{t_{i+1}}L(s)^2\inte s,
}
where the first bound follows by Jensen's inequality, the second by \cref{lem:v-rep} and It\^o's isometry, and the third by \eqref{score-lip}. 
Hence
\ba{
\sum_{i=0}^{N-1}\E[|\Delta_i|^2]
\leq d\eta^2\int_{t_0}^{t_N}L(s)^2\inte s.
}
Similarly, 
\ben{\label{eq:initialize}
\|X_{t_0}-\eul_{t_0}\|_2
\leq t_0|\mu-\mean|+\norm{\int_0^{t_0}v_r\inte r-t_0\mu}_2
\leq t_0|\mu-\mean|+\sqrt dt_0\sqrt{\int_0^{t_0}L(s)^2\inte s}.
}
Inserting these bounds into \eqref{eq:master} gives the desired result. 
\qed

\subsection{Proof of Theorem \ref{thm:main2}}\label{pr:main2}

Under \eqref{eq:sampling}, we have the following general bound for the discretization error.
\begin{lemma}\label{delta-sum-bound}
Assume \eqref{eq:sampling} for some $\eta>0$. Then, 
\ben{\label{eq:delta-sum}
\sum_{i=0}^{N-1}\|\Delta_i\|_2^2\leq d\eta^2\log(1/\delta)+(\trace(\Sigma)\vee d)\eta^2.
}
\end{lemma}

\begin{proof}
By Jensen's inequality and the martingale property of $(v_r)_{r\in(0,1)}$, for any $i\in\indices$,
\ben{\label{delta-jensen}
\|\Delta_i\|_2^2\leq h_i\int_{t_i}^{t_{i+1}}\E[|v_r-v_{t_i}|^2]\inte r
\leq h_i^2\bra{\E[|v_{t_{i+1}}|^2]-\E[|v_{t_{i}}|^2]}.
} 
Hence,
\ba{
\sum_{i=0}^{N-1}\|\Delta_i\|_2^2
&\leq\sum_{i=0}^{N-1}h_i^2\bra{\E[|v_{t_{i+1}}|^2]-\E[|v_{t_{i}}|^2]}\\
&\leq\eta^2\sum_{i=0}^{N-1}(1-t_{i+1})\bra{\E[|v_{t_{i+1}}|^2]-\E[|v_{t_{i}}|^2]}\\
&=\eta^2\sum_{i=1}^{N}(1-t_{i})\E[|v_{t_{i}}|^2]
-\eta^2\sum_{i=0}^{N-1}(1-t_{i+1})\E[|v_{t_{i}}|^2]\\
&\leq\eta^2\sum_{i=1}^{N-1}h_i\E[|v_{t_{i}}|^2]
+\eta^2(1-t_{N})\E[|v_{t_{N}}|^2].
}
By \cref{lem:v2},
\ba{
\sum_{i=1}^{N-1}h_i\E[|v_{t_{i}}|^2]
\leq d\sum_{i=1}^{N-1}\frac{h_i}{1-t_i}+(\trace(\Sigma)-d)(t_N-t_0)
\leq d\log(1/\delta)+(\trace(\Sigma)-d)(t_N-t_0)
}
and
\ba{
(1-t_N)\E[|v_{t_{N}}|^2]
\leq d+(\trace(\Sigma)-d)(1-t_N)
}
Hence,
\ba{
\sum_{i=0}^{N-1}\|\Delta_i\|_2^2
\leq d\eta^2\log(1/\delta)+\{d+(\trace(\Sigma)-d)(1-t_0)\}\eta^2
\leq d\eta^2\log(1/\delta)+(\trace(\Sigma)\vee d)\eta^2.
}
This completes the proof.
\end{proof}

To control the initialization error, we need the following elementary matrix result, which is an immediate consequence of \cite[Corollary 7.7.4(c)]{HoJo13}. 
\begin{lemma}\label{lem:frob}
Let $A,B$ and $C$ be three $d\times d$ symmetric matrices such that $A\preceq B\preceq C$. Then $\|B\|_F\leq\|A\|_F\vee\|C\|_F$.
\end{lemma}

\begin{proof}[\bf\upshape Proof of \cref{thm:main2}]
Under the assumptions of the theorem, \eqref{eq:master} holds by \cref{lem:master}. 
By Jensen's inequality, \cref{lem:v-rep} and It\^o's isometry,
\ben{\label{init-bound}
\norm{\int_0^{t_0}v_r\inte r-t_0\mu}_2^2
\leq t_0\int_0^{t_0}\E[|v_r-v_0|^2]\inte r
\leq t_0^2\int_0^{t_0}\E[\|\nabla v(s,X_s)\|_F^2]\inte s.
}
Since \eqref{score-lip-one} and \eqref{eq:v-deriv} yield
\ba{
-\frac{I_d}{1-s}\preceq\nabla v(s,X_s)\preceq L(s)I_d,
}
we obtain by \cref{lem:frob}
\ba{
\norm{\int_0^{t_0}v_r\inte r-t_0\mu}_2^2
\leq t_0^2\int_{0}^{t_{0}}d\max\cbra{L(s)^2,\,\frac{1}{(1-s)^2}}\inte s
\leq dt_0^2\bra{4t_0\vee\int_{0}^{t_{0}}L(s)^2\inte s},
}
where we used $t_0\leq1/2$ to deduce the last inequality. 
Hence
\ba{
\|X_{t_0}-\eul_{t_0}\|_2
\leq t_0|\mu-\mean|+\sqrt dt_0\bra{2\sqrt{t_0}\vee\sqrt{\int_{0}^{t_{0}}L(s)^2\inte s}}.
}
Inserting this bound, \eqref{eq:delta-sum} and $t_N^{-1/2}\leq\sqrt2$ into \eqref{eq:master} gives the desired result.
\end{proof}

\subsection{Proof of Theorem \ref{thm:main3}}\label{pr:main3}

The proof is a minor variant of the preceding arguments. 
By \eqref{wass-es-est}, $t_N^{-1/2}\leq\sqrt2$ and \cref{lem:es},
\ben{\label{wass-es-est-3}
\wass_2(\ddpm_N,\,\target)
\leq \sqrt2\|\euler_{t_N}-X_{t_N}\|_2+\sqrt{\trace(\Sigma)\delta^2+d\delta}.
}
We bound the first term on the right-hand side.
For every $i\in\indices$, 
\ba{
X_{t_{i+1}}-\euler_{t_{i+1}}
&=\bra{1+\frac{h_i}{t_i}}(X_{t_{i}}-\euler_{t_{i}})
+\frac{h_i}{\sqrt{t_i}}\bra{\nabla\log\pi_{t_i}\bra{\frac{X_{t_i}}{\sqrt{t_i}}}-\scored_i\bra{\frac{\euler_{t_i}}{\sqrt{t_i}}}}
+\Delta_i\\
&=\mathbb{I}_i+\mathbb{II}_i+\Delta_i,
}
where
\ba{
\mathbb{I}_i&:=\bra{1+\frac{h_i}{t_i}}(X_{t_{i}}-\euler_{t_{i}})
+\frac{h_i}{\sqrt{t_i}}\bra{\scored_i\bra{\frac{X_{t_i}}{\sqrt{t_i}}}-\scored_i\bra{\frac{\euler_{t_i}}{\sqrt{t_i}}}},\\
\mathbb{II}_i&:=\frac{h_i}{\sqrt{t_i}}\bra{\nabla\log\pi_{t_i}\bra{\frac{X_{t_i}}{\sqrt{t_i}}}-\scored_i\bra{\frac{X_{t_i}}{\sqrt{t_i}}}}.
}
Since $\E[\Delta_i\mid\mcl F_{t_i}]=0$ and $\mathbb{I}_i+\mathbb{II}_i$ is $\mcl F_{t_i}$-measurable, we have 
\be{
\|X_{t_{i+1}}-\euler_{t_{i+1}}\|_p^2
=\|\mathbb{I}_i+\mathbb{II}_i\|_2^2+\|\Delta_i\|_2^2.
}
We can rewrite $\mathbb{I}_i$ as
\ba{
\mathbb{I}_i=\int_0^1\cbra{I_d+\frac{h_i}{t_i}\bra{I_d+\nabla\scored_i\bra{\frac{(1-\theta)\euler_{t_i}+\theta X_{t_i}}{\sqrt{t_i}}}}}(X_{t_i}-\euler_{t_i})\inte\theta.
}
By \eqref{scored-lip} and $h_i\leq1-t_i$, we have for almost all $y\in\mathbb R^d$ with respect to the Lebesgue measure
\ba{
(1+h_iL(t_i))I_d\succeq I_d+\frac{h_i}{t_i}\bra{I_d+\nabla\scored_i(y)}
\succeq 
\bra{1-\frac{2h_i}{1-t_i}}I_d\succeq -I_d.
}
Therefore, $|\mathbb{I}_i|\leq(1+h_iL(t_i))|X_{t_i}-\euler_{t_i}|$. 
Also, $\|\mathbb{II}_i\|_2\leq\frac{h_i}{\sqrt{t_i}}\eps_{\mathrm{score}}(t_i)$ by \eqref{score-error-p}. 
Hence, with $x_i:=\|X_{t_i}-\euler_{t_i}\|_2$, we have
\be{
x_{i+1}^2\leq \bra{\bra{1+h_iL(t_i)}x_i+\frac{h_i}{\sqrt{t_i}}\eps_{\mathrm{score}}(t_i)}^2+\|\Delta_i\|_2^2.
}
Therefore, applying \cref{lem:recursion} with $A_i:=\log(1+h_iL(t_i)),B_i:=\frac{h_i}{\sqrt{t_i}}\eps_{\mathrm{score}}(t_i)$ and $C_i:=\|\Delta_i\|_2$, we obtain
\be{
x_N\leq e^{\bar A_{N-1}}x_{0}
+\sum_{i=0}^{N-1}e^{\bar A_{N-1}-\bar A_i}\eps_\text{score}(t_i)\frac{h_i}{\sqrt{t_i}}
+\sqrt{\sum_{i=0}^{N-1}e^{2(\bar A_{N-1}-\bar A_i)}\|\Delta_i\|_2^2},
}
where $\bar A_i:=A_0+\cdots+A_i$. 
Noting $\log x\leq x-1$ for any $x>0$, we obtain by \eqref{eq:p-quad} and \cref{delta-sum-bound}
\ben{\label{rec-applied-3}
x_N\leq e^{\int_{t_0}^{t_N}L(s)\inte s}\bra{x_0+2\int_{t_0}^{t_N}\frac{\eps_\mathrm{score}(s)}{\sqrt s}\inte s
+\sqrt{d\log(1/\delta)+(\trace(\Sigma)\vee d)}\,\eta}.
}
For any $r\in(0,t_0]$, we have by \cref{lem:v2} and $t_0\leq1/2$,
\ba{
\|v_r-\mu\|_2^2=\trace(\Cov[v_r])\leq\E[|v_r|^2]\leq\frac{d}{1-r}+\trace(\Sigma)-d\leq\trace(\Sigma)+d.
}
Therefore, $x_0\leq t_0(|\mu-\mean|+\sqrt{\trace(\Sigma)+d})$ by \eqref{eq:initialize} and \eqref{init-bound}. Inserting this into \eqref{rec-applied-3} gives the desired bound for the first term on the right-hand side of \eqref{wass-es-est-3}. \qed

\subsection{Proof of Theorem \ref{lem:LS}}\label{pr-lem:LS}

We need to prove that the law of $t_N^{-1/2}\euler_{t_N}$ satisfies $LS(2e^{2\int_{t_0}^{t_N}L(s)\inte s})$. 
First, since $\euler_{t_0}\sim N(t_0\mean,t_0I_d)$, the law of $\euler_{t_0}$ satisfies $LS(t_0)$. 
Next, suppose that the law of $\euler_{t_i}$ satisfies $LS(c_i)$ for some $i\in\indices$ and $c_i>0$. 
Since the map $x\mapsto x+\frac{h_i}{\sqrt{t_i}}\scored_i(x/\sqrt{t_i})$ is $(1+h_iL(t_i))$-Lipschitz continuous by the proof of \cref{thm:main3}, the law of $\euler_{t_{i}}+\frac{h_i}{\sqrt{t_i}}\scored_i(\euler_{t_i}/\sqrt{t_i})$ satisfies $LS((1+h_iL(t_i))^2c_i)$ by \cite[Lemma 16]{VeWi23}. 
Moreover, since $\euler_{t_{i+1}}=\euler_{t_{i}}+\frac{h_i}{\sqrt{t_i}}\scored_i(\euler_{t_i}/\sqrt{t_i})+W_{t_{i+1}}-W_{t_i}$, Lemma 17 in \cite{VeWi23} implies that the law of $\euler_{t_{i+1}}$ satisfies $LS(c_{i+1})$ for a constant $c_{i+1}>0$ such that 
\ben{\label{rec-LS}
c_{i+1}\leq(1+h_iL(t_i))^2c_i+h_i.
}
Therefore, by induction, the law of $\euler_{t_i}$ satisfies $LS(c_i)$ for all $i=0,1,\dots,N$, where $(c_i)_{i=0}^N$ satisfy \eqref{rec-LS} and $c_0=t_0$. 
By \cref{lem:recursion} and the inequality $\sum_{i=0}^{N-1}\log(1+h_iL(t_i))\leq\sum_{i=0}^{N-1}h_iL(t_i)\leq\int_{t_0}^{t_N}L(s)\inte s$,
\[
c_N\leq e^{2\int_{t_0}^{t_N}L(s)\inte s}t_0+\sum_{i=0}^{N-1}e^{2\int_{t_0}^{t_N}L(s)\inte s}h_i\leq e^{2\int_{t_0}^{t_N}L(s)\inte s}.
\]
Noting $t_N^{-1/2}\leq\sqrt 2$, this yields the desired result.\qed

\subsection{Proof of Theorem \ref{prop:lc}}\label{pr-prop:lc}


%
%

To control the initialization and discretization errors, we use the following lemma. 
\begin{lemma}\label{V-bound}
If $\target$ is log-concave, then
\ben{\label{v-op-bound}
\norm{\|V_t\|_{op}}_p\lesssim p\|\Sigma\|_{op}
}
for all $0<t\leq (2^{15}\Lambda\log^2(d+1))^{-1}$ and $p\geq2$. Moreover,
\ben{\label{v-fr-bound}
\E[\|V_t\|_{F}^2]\lesssim \Lambda^2d
}
for all $0<t<1$.
\end{lemma}

When $\Sigma=I_d$, \eqref{v-op-bound} and \eqref{v-fr-bound} follow from \cite[Theorem 7.10]{KlLe25ams} and \cite[Lemma 2.1]{Gu24}, respectively, via a time change (see \cite[Lemma 3.6]{FaKo24} for the former). 
By slightly modifying their proofs, we can easily extend these results to the general covariance case. We outline the proof in \cref{sec:V-bound}.

\begin{lemma}\label{v-bound-lc-l2}
If $\target$ is log-concave, then
\[
\|v_t-v_s\|_2\lesssim \Lambda\sqrt d\frac{\sqrt{t-s}}{(1-t)^2}
\]
for any $0\leq s\leq t<1$, where we set $v_0:=\mu$.
\end{lemma}

\begin{proof}
By \cref{lem:v-rep} and It\^o's isometry,
\ba{
\|v_t-v_s\|_2^2=\int_s^t\E[\|\nabla v(r,X_r)\|_F^2]\inte r.
}
Also, by \eqref{eq:v-deriv}, we have for any $0<s<1$
\ben{\label{eq:v-deriv-lc}
-\frac{I_d}{1-s}\preceq\nabla v(s,X_s)\preceq \frac{\Gamma_s}{1-s}.
}
Thus, by \cref{lem:frob} and \eqref{v-fr-bound},
\ba{
\|v_t-v_s\|_2^2\leq\int_s^t\frac{\max\{\E[\|\Gamma_r\|_F^2],\,d\}}{(1-r)^2}\inte r
\lesssim\Lambda^2d\int_s^t\frac{\inte r}{(1-r)^4}
\leq\Lambda^2d\frac{t-s}{(1-t)^4}.
}
This completes the proof. 
\end{proof}

We also need the following technical lemma.
\begin{lemma}\label{lem:lambda-bound-lc}
Assume $\max_{i\in\indices}\beta_i\leq3/4$ and \eqref{eq:sampling3} for some $\eta>0$. 
Then, for any $i\in\indices$,
\[
\sum_{j=i}^{N-1}\frac{h_j}{t_j}\leq128+\log\frac{1}{t_i}.
\]
\end{lemma}

\begin{proof}
Recall that $\max_{i\in\indices}\beta_i\leq3/4$ implies $t_{i+1}\leq4t_i$ for every $i\in\indices$. 
Hence,
\ba{
\sum_{j=i}^{N-1}h_j\bra{\frac{1}{t_j}-\frac{1}{t_{j+1}}}
=\sum_{j=i}^{N-1}\frac{h_j^2}{t_jt_{j+1}}
\leq4\sum_{j=i}^{N-1}\frac{h_j^2}{t_{j+1}^2}
\leq4\eta\sum_{j=i}^{N-1}\frac{h_j}{t_{j+1}^{3/2}}
\leq4\eta\int_{t_i}^{t_N}\frac{\inte s}{s^{3/2}}
\leq\frac{8\eta}{\sqrt{t_0}}
\leq128,
}
where the second and the last inequalities follow from \eqref{eq:sampling3}. 
Consequently,
\ba{
\sum_{j=i}^{N-1}\frac{h_j}{t_j}
\leq128+\sum_{j=i}^{N-1}\frac{h_j}{t_{j+1}}
\leq128+\int_{t_i}^{t_N}\frac{\inte s}{s}
\leq128+\log\frac{1}{t_i}.
}
This completes the proof. 
\end{proof}

\begin{proof}[\bf\upshape Proof of \cref{prop:lc}]
By \cite[Proposition 3.5]{SaWe14}, $\pi_t$ is log-concave for all $t\in[0,1)$. Hence,
\ben{\label{eq:lip-lc}
\nabla^2\log f_{t}(x)\preceq I_d\quad\text{for all }0\leq t<1\text{ and }x\in\mathbb R^d.
}
Therefore, \cref{lem:master} implies that \eqref{eq:master0} holds with $\lambda_{i,N}:=\sum_{j=i}^{N-1}h_j/t_j$. 
By \cref{lem:lambda-bound-lc}, $\lambda_{i,N}\leq8+\log(1/t_i)$. 
Hence,
\ban{
&\sqrt{t_N}\bra{\wass_2(\ddpm_N,\,\target)-\sqrt{\trace(\Sigma)\delta^2+d\delta}}
\notag\\
&\lesssim\frac{1}{t_0}\|X_{t_0}-\eul_{t_0}\|_2
+\sqrt{\sum_{i=0}^{N-1}t_{i+1}^{-2}\|\Delta_i\|_2^2}
+\sum_{i=0}^{N-1}\eps_\text{score}(t_i)\frac{h_i}{t_{i+1}\sqrt{t_i}}
\notag\\
&=:\mathbb{I}+\mathbb{II}+\mathbb{III}.\label{lc-applied}
}
First we bound $\mathbb I$. 
By the integral Minkowski inequality and \cref{v-bound-lc-l2},
\be{
\norm{\int_0^{t_0}(v_r-\mu)\inte r}_2
\lesssim t_0\Lambda\sqrt d\sqrt{\frac{t_0}{1-t_0}}
\lesssim\Lambda\sqrt d t_0^{3/2},
}
where the last bound follows from $t_0\leq1/2$. 
Hence,
\ben{\label{I-bound}
\mathbb{I}\lesssim |\mu-\mean|+\Lambda \sqrt{dt_0}.
}
Next we bound $\mathbb{II}$. 
With $k:=\max\{i:t_i\leq1/2\}$, we have
\ba{
\mathbb{II}^2
=\sum_{i=0}^{k-1}\frac{\|\Delta_i\|_2^2}{t_{i+1}^{2}}
+\sum_{i=k}^{N-1}\frac{\|\Delta_i\|_2^2}{t_{i+1}^{2}}
=:\mathbb{II}_1+\mathbb{II}_2.
}
By the integral Minkowski inequality and \cref{v-bound-lc-l2},
\ben{
\|\Delta_i\|_2\lesssim h_i\Lambda\sqrt d\frac{\sqrt{h_i}}{(1-t_{i+1})^2}
\lesssim\Lambda\sqrt d h_i^{3/2}
}
for every $i=0,1,\dots,k-1$, where we used $t_k\leq1/2$ for the last inequality. 
Hence,
\ba{
\mathbb{II}_1
&\lesssim \Lambda^2d\sum_{i=0}^{k-1}\frac{h_i^3}{t_{i+1}^2}
\leq\Lambda^2d\eta^2\sum_{i=0}^{k-1}\frac{h_i}{t_{i+1}}
\leq\Lambda^2d\eta^2\log(1/t_0),
}
where the second inequality follows from \eqref{eq:sampling3}. 
Meanwhile, by the definition of $k$ and \cref{delta-sum-bound}
\ba{
\mathbb{II}_2\leq4\sum_{i=k}^{N-1}\|\Delta_i\|_2^2
\leq 4d\eta^2\log(1/\delta)+4\eta^2(\trace(\Sigma)\vee d)
\leq4d\eta^2\log(1/\delta)+4\Lambda d\eta^2.
}
Consequently,
\ben{\label{II-bound}
\mathbb{II}\lesssim \bra{\Lambda\sqrt{d\log(1/t_0)}+\sqrt{d\log(1/\delta)}}\eta.
}
Finally, 
\ben{\label{III-bound}
\mathbb{III}\leq2\sum_{i=0}^{N-1}\eps_\text{score}(t_i)\frac{h_i}{t_{i+1}^{3/2}}
\leq2\int_{t_0}^{t_N}\frac{\eps_\text{score}(s)}{s^{3/2}}\inte s.
}
Combining \eqref{lc-applied}, \eqref{I-bound}, \eqref{II-bound} and \eqref{III-bound} with $t_N^{-1/2}\leq\sqrt 2$ gives the desired result.
\end{proof}

\subsection{Proof of Theorem \ref{thm:lc}}\label{pr-thm:lc}

To control the initialization error more precisely, we first derive a preliminary bound for $\|X_{t_k}-\eul_{t_k}\|_4$ when $t_k$ is small. 
For this purpose, we need a Burkholder--Davis--Gundy type inequality for continuous local martingales in $\mathbb R^d$ with a dimension-free constant. 
Such a result can easily be derived from a univariate counterpart via \citet{KaSz91}'s trick (see \cite[Lemma 2.2]{SeSo03}). We use the following version:
\begin{lemma}\label{bdg}
Let $M=(M_t)_{t\in[0,1]}$ be a continuous local martingale in $\mathbb R^d$ such that $M_0=0$ a.s. 
Then, for any $t\in[0,1]$ and $p\geq2$,
\[
\|M_t\|_p\leq4\sqrt p\norm{\trace(\langle M,M\rangle_t)}_{p/2}^{1/2},
\]
where $\langle M,M\rangle_t:=(\langle M_j,M_k\rangle_t)_{1\leq j,k\leq d}$ is the predictable quadratic covariation matrix of $M$ on $[0,t]$. 
\end{lemma}

\begin{proof}
By \cite[Theorem 3.1]{KaSz91}, there exists a continuous local martingale $N=(N_t)_{t\in[0,1]}$ in $\mathbb R^2$ such that $|M_t|=|N_t|$ and $\trace(\langle M,M\rangle_t)=\trace(\langle N,N\rangle_t)$ a.s.~for all $t\in[0,1]$. Hence
\ba{
\|M_t\|_p&=\|N_t\|_p\leq\|N^1_{t}\|_p+\|N^2_{t}\|_p
\leq2\sqrt p\bra{\|\langle N^1,N^1\rangle_t^{1/2}\|_p+\|\langle N^2,N^2\rangle_t^{1/2}\|_p}\\
&\leq4\sqrt p\norm{\sqrt{\trace(\langle N,N\rangle_t)}}_p
=4\sqrt p\norm{\trace(\langle M,M\rangle_t)}_{p/2}^{1/2},
}
where the second inequality follows by \cite[Theorem A]{CaKr91}. 
\end{proof}

\begin{lemma}\label{v-bound-lc}
If $\target$ is log-concave, then
\[
\|v_t-v_s\|_p\lesssim p^{3/2}\Lambda\sqrt{d(t-s)}
\]
for any $0\leq s\leq t\leq(2^{15}\Lambda\log^2(d+1))^{-1}$ and $p\geq2$, where we set $v_0:=\mu$.
\end{lemma}

\begin{proof}
By Lemmas \ref{lem:v-rep} and \ref{bdg},
\ba{
\|v_t-v_s\|_p\leq4\sqrt{p}\norm{\int_s^t\|\nabla v(r,X_r)\|_F^2\inte r}_{p/2}^{1/2}.
}
Thus, by \eqref{eq:v-deriv-lc} and \cref{lem:frob},
\ba{
\|v_t-v_s\|_p\leq4\sqrt{p}\norm{\int_s^t\frac{\max\{\|\Gamma_r\|_F^2,\,d\}}{(1-r)^2}\inte r}_{p/2}^{1/2}.
}
Hence, the integral Minkowski inequality and \eqref{v-op-bound} yield
\ba{
\|v_t-v_s\|_p\leq4\sqrt{p}\bra{\int_s^t\frac{\max\{\norm{\|\Gamma_r\|_F}_{p}^2,\,d\}}{(1-r)^2}\inte r}^{1/2}
\lesssim p^{3/2}\sqrt d\Lambda\sqrt{\int_s^t\frac{\inte r}{(1-r)^4}}.
}
Since $t\leq1/2$, this gives the desired result. 
\end{proof}

\begin{lemma}\label{lem:apriori}
Under the assumptions of \cref{thm:lc}, for any $k\in\{0,1,\dots,N-1\}$ with $t_k\leq(2^{15}\Lambda\log^2(d+1))^{-1}$,
\[
\|X_{t_k}-\eul_{t_k}\|_4\lesssim t_k\bra{|\mu-\mean|+\sqrt d\Lambda\sqrt{t_0}
+\Lambda\sqrt{d\log(1/t_0)}\,\eta}.
\]
\end{lemma}

\begin{proof}
Since \eqref{eq:lip-lc} holds under the assumptions of \cref{thm:lc}, we have by Lemmas \ref{lem:euler} and \ref{lem:lambda-bound-lc} 
\ben{\label{master-applied-lc4}
\|X_{t_k}-\eul_{t_k}\|_4\lesssim 
\frac{t_k}{t_0}\|X_{t_0}-\eul_{t_0}\|_4
+\sqrt{\sum_{i=0}^{k-1}\frac{t_k^2}{t_{i+1}^2}\|\Delta_i\|_4^2}.
}
By the integral Minkowski inequality and \cref{v-bound-lc},
\ben{\label{v0-lc4}
\norm{\int_0^{t_0}(v_r-\mu)\inte r}_4
\lesssim\Lambda\sqrt d t_0^{3/2}
}
and
\ben{\label{delta-lc4}
\|\Delta_i\|_4\lesssim\Lambda\sqrt d h_i^{3/2}
}
for every $i=0,1,\dots,k-1$.
\eqref{v0-lc4} gives
\ba{
\|X_{t_0}-\eul_{t_0}\|_4\lesssim t_0|\mu-\mean|+\Lambda\sqrt d t_0^{3/2}.
}
Moreover, by \eqref{delta-lc4} and \eqref{eq:sampling3},
\ba{
\sum_{i=0}^{k-1}\frac{t_k^2}{t_{i+1}^2}\|\Delta_i\|_4^2
\lesssim \Lambda^2dt_k^2\eta^2\sum_{i=0}^{k-1}\frac{h_i}{t_{i+1}}
\leq \Lambda^2dt_k^2\eta^2\log(1/t_0).
}
Inserting these bounds into \eqref{master-applied-lc4} gives the desired result.
\end{proof}




Next we bound $\|X_{t_{N}}-\eul_{t_{N}}\|_2$. 
We need the following auxiliary result. 
\begin{lemma}\label{lem:v-hessian}
Suppose that $\target$ is log-concave. 
Then, for all $0<t<1$ and $x,a\in\mathbb R^d$,
\ben{\label{eq:v-hessian}
\sqrt{\sum_{j=1}^d\langle\nabla^2m_j(t,x),\,a^{\otimes2}\rangle^2}\lesssim \frac{|a|^2}{t^{3/2}\sqrt{1-t}}.
}
\end{lemma}

\begin{proof}
Observe that the left-hand side of \eqref{eq:v-hessian} is equal to
\[
\sup_{u\in\mathbb R^d:|u|=1}\sum_{j=1}^d\langle\nabla^2m_j(t,x),\,a^{\otimes2}\rangle u_j.
\]
Therefore, it suffices to show
\[
\sum_{j=1}^d\langle\nabla^2m_j(t,x),\,a^{\otimes2}\rangle u_j\leq C\frac{|a|^2}{t^{3/2}\sqrt{1-t}}
\]
for all unit vector $u$ in $\mathbb R^d$. 
Using \eqref{eq:m-deriv} and \cref{sl-deriv}, we can verify 
\ba{
\nabla^2 m_j(t,x)
&=\frac{1}{(1-t)^2}\int_{\mathbb R^d}\bra{y_j-m_j(t,x)}\bra{y-m(t,x)}^{\otimes2}\target_{t,x}(\mathrm{d}y)
=\frac{\E[Y_jY^{\otimes2}]}{(1-t)^2},
}
where $Y:=Y^*-\E[Y^*]$ with $Y^*\sim\target_{t,x}$. 
Observe that $\target_{t,x}$ is $\frac{t}{1-t}$-strongly log-concave by construction. 
Hence,
\ba{
\sum_{j=1}^d\langle \nabla^2m_j(t,x),\,a^{\otimes2}\rangle u_j
&=\frac{\langle\E[(Y\cdot u)Y^{\otimes2}],\,a^{\otimes2}\rangle}{(1-t)^2}
\lesssim\frac{\|V(t,x)\|_{op}^{3/2}|a|^2}{(1-t)^2}
\leq\frac{|a|^2}{t^{3/2}\sqrt{1-t}},
}
where the first inequality is by \cite[Lemma 7.6]{KlLe25ams} and the second by the Brascamp--Lieb inequality (see e.g.~\cite[Theorem 4.4]{KlLe25ams}). 
\end{proof}

\begin{lemma}\label{lem:lc-main}
Under the assumptions of \cref{thm:lc},
\ba{
\|X_{t_{N}}-\eul_{t_{N}}\|_2
&\lesssim\sqrt{\Lambda t_0}|\mu-\mean|
+\Lambda^{-1/4}t_0^{1/4}|\mu-\mean|^2\\
&\quad+\Lambda^{3/2}\sqrt dt_0
+\bra{\Lambda\sqrt{d\log(1/t_0)}+\sqrt{d\log(1/\delta)}}\eta.
}
\end{lemma}

\begin{proof}
For every $i\in\indices$, consider the decomposition \eqref{eq:basic-decomp}. 
Since $R_i$ is $\mcl F_{t_i}$-measurable and $\E[\Delta_i\mid\mcl F_{t_i}]=0$, we have
\ben{\label{eq:pyt}
\|X_{t_{i+1}}-\eul_{t_{i+1}}\|_2^2
=\|R_i\|_2^2+\|\Delta_i\|_2^2.
} 
\eqref{eq:psd} and \eqref{eq:lip-lc} give
\ben{\label{eq:lip-lc-applied}
\|R_i\|_2\leq\bra{1+\frac{h_i}{t_i}}\|X_{t_{i}}-\eul_{t_{i}}\|_2.
}
To obtain a better bound when $t_i$ is small, consider the case $t_i\leq(2^{15}\Lambda\log^2(d+1))^{-1}$ and rewrite $R_i$ as
\ba{
R_i=\bra{1-\frac{h_i}{1-t_i}}(X_{t_{i}}-\eul_{t_{i}})+h_i\frac{m(t_i,X_{t_i})-m(t_i,\eul_{t_i})}{1-t_i}.
}
Since $h_i\leq1-t_i$ and $t_i\leq1/2$,
\ben{\label{r-est-m}
\|R_i\|_2\leq\|X_{t_{i}}-\eul_{t_{i}}\|_2+2h_i\|m(t_i,X_{t_i})-m(t_i,\eul_{t_i})\|_2.
}
For $j=1,\dots,d$, denote by $m_j(t,x)$ the $j$-th component of $m(t,x)$. 
By Taylor's theorem, 
\ba{
m_j(t_i,X_{t_i})-m_j(t_i,\eul_{t_i})
&=\langle\nabla m_j(t_i,X_{t_i}),\,X_{t_i}-\eul_{t_i}\rangle\\
&\quad+\int_0^1(1-\theta)\langle\nabla^2 m_j\bra{t_i,X_{t_i}+\theta(X_{t_i}-\eul_{t_i})},\,(X_{t_i}-\eul_{t_i})^{\otimes2}\rangle\inte\theta.
}
Combining this with \cref{lem:v-hessian} gives
\ba{
\abs{m(t_i,X_{t_i})-m(t_i,\eul_{t_i})}
\lesssim\|\nabla m(t_i,X_{t_i})\|_{op}|X_{t_i}-\eul_{t_i}|+t_i^{-3/2}|X_{t_i}-\eul_{t_i}|^2.
}
Noting that $\nabla m(t_i,X_{t_i})=\Gamma_{t_i}$ by \eqref{eq:m-deriv}, we obtain
\ba{
\|m(t_i,X_{t_i})-m(t_i,\eul_{t_i})\|_2
\lesssim\norm{\|\Gamma_{t_i}\|_{op}}_4\|X_{t_i}-\eul_{t_i}\|_4+t_i^{-3/2}\|X_{t_i}-\eul_{t_i}\|_4^2.
}
Hence, \eqref{v-op-bound} and \cref{lem:apriori} yield
\ba{
\|m(t_i,X_{t_i})-m(t_i,\eul_{t_i})\|_2
&\lesssim \Lambda t_ib
+\sqrt{t_i}b^2,
}
where $b:=|\mu-\mean|+\sqrt d\Lambda\sqrt{t_0}+\Lambda\sqrt{d\log(1/t_0)}\,\eta$. 
Inserting this bound into \eqref{r-est-m} gives
\ben{\label{eq:r-bound-lc}
\|R_i\|_2\leq\|X_{t_{i}}-\eul_{t_{i}}\|_2+Ch_i(\Lambda t_ib+\sqrt{t_i}b^2),
}
where $C>0$ is a universal constant. 

Now, we fix $k\in\{0,1,\dots,N-1\}$ with $t_k\leq(2^{15}\Lambda\log^2(d+1))^{-1}$, which is specified later (see \eqref{somewhere}). 
Then, combining \eqref{eq:lip-lc-applied} and \eqref{eq:r-bound-lc} with \eqref{eq:pyt} gives  
\ba{
\|X_{t_{i+1}}-\eul_{t_{i+1}}\|_2^2
\leq\bra{e^{A_i}\|X_{t_{i}}-\eul_{t_{i}}\|_2+B_i}^2+\|\Delta_i\|_2^2
\quad\text{for all }i\in\indices,
}
where
\ba{
A_i:=\begin{cases}
0 & \text{if }i< k,\\
\log(1+h_i/t_i) & \text{if }i\geq k,
\end{cases}
\qquad
B_i:=\begin{cases}
Ch_i(\Lambda t_ib+\sqrt{t_i}b^2) & \text{if }i< k,\\
0 & \text{if }i\geq k.
\end{cases}
}
Hence, \cref{lem:recursion} gives
\ba{
\|X_{t_{N}}-\eul_{t_{N}}\|_2
\leq e^{\bar A_{N-1}}\|X_{t_{0}}-\eul_{t_{0}}\|_2
+\sum_{i=0}^{N-1}e^{\bar A_{N-1}-\bar A_i}B_i
+\sqrt{\sum_{i=0}^{N-1}e^{2(\bar A_{N-1}-\bar A_i)}\|\Delta_i\|_2^2},
}
where $\bar A_i:=\sum_{j=0}^iA_j$.
Using \cref{lem:lambda-bound-lc}, we conclude
\ba{
\|X_{t_{N}}-\eul_{t_{N}}\|_2
\lesssim t_k^{-1}\|X_{t_{0}}-\eul_{t_{0}}\|_2
+\Lambda t_kb+\sqrt{t_k}b^2
+\sqrt{\sum_{i=0}^{N-1}\frac{\|\Delta_i\|_2^2}{t_{i+1}^2}}.
}
Inserting \eqref{I-bound} and \eqref{II-bound} into the right-hand side, we obtain
\ben{\label{rec-applied-lc}
\|X_{t_{N}}-\eul_{t_{N}}\|_2
\lesssim
\frac{t_0}{t_k}\bra{|\mu-\mean|+\Lambda \sqrt{dt_0}}
+\Lambda t_kb+\sqrt{t_k}b^2
+\bra{\Lambda\sqrt{d\log(1/t_0)}+\sqrt{d\log(1/\delta)}}\eta.
}
Now, let
\ben{\label{somewhere}
k:=\max\{i\in\{0,1,\dots,N\}:t_i\leq\sqrt{t_0/\Lambda}\}.
}
Taking $c=2^{-30}$ in the assumptions of \cref{thm:lc}, we have $t_k\leq\sqrt{t_0/\Lambda}\leq(2^{15}\Lambda\log^2(d+1))^{-1}$. 
Hence,
\ben{\label{sw-applied1}
\frac{t_0}{t_k}\bra{|\mu-\mean|+\Lambda \sqrt{dt_0}}
\lesssim\sqrt{\Lambda t_0}|\mu-\mean|+\Lambda^{3/2}\sqrt dt_0
}
and
\ba{
\Lambda t_kb
&\leq\sqrt{\Lambda t_0}|\mu-\mean|+\Lambda^{3/2}\sqrt{d}t_0+\Lambda^{3/2}\sqrt{dt_0\log(1/t_0)}\,\eta,\\
\sqrt{t_k}b^2
&\leq3\bra{\Lambda^{-1/4}t_0^{1/4}|\mu-\mean|^2+\Lambda^{7/4}dt_0^{5/4}+\Lambda^{7/4}dt_0^{1/4}\log(1/t_0)\eta^2}.
}
Since $t_0\leq(\Lambda d^2)^{-1}$ and $\sqrt{\Lambda d\log(1/t_0)}\eta\leq1$ by assumption, we conclude
\ben{\label{sw-applied2}
\Lambda t_kb+\sqrt{t_k}b^2
\lesssim \sqrt{\Lambda t_0}|\mu-\mean|
+\Lambda^{-1/4}t_0^{1/4}|\mu-\mean|^2
+\Lambda^{3/2}\sqrt{d}t_0
+\Lambda\sqrt{d\log(1/t_0)}\,\eta.
}
Inserting \eqref{sw-applied1} and \eqref{sw-applied2} into \eqref{rec-applied-lc} completes the proof. 
\end{proof}

\begin{proof}[\bf\upshape Proof of \cref{thm:lc}]
The desired result follows from \cref{lem:lc-main}, \eqref{III-bound} and \cref{lem:es}.
\end{proof}

\appendix

\section{Appendix: Additional proofs}\label{sec:appendix}

\subsection{Proof of Proposition \ref{prop:lb}}\label{pr-prop:lb}

Since $\target=N(\mu,\sigma^2I_d)$, we have
\ba{
\nabla\log\pi_t(y)=-\frac{y-\sqrt t\mu}{t\sigma^2+1-t}
\quad\text{for all }0<t<1\text{ and }y\in\mathbb R^d. 
}
Hence, for all $i=0,1,\dots,N-1$,
\ba{
\sqrt{t_{i+1}}\ddpm_{i+1}
&=\frac{t_{i+1}}{\sqrt{t_i}}\bra{\ddpm_i-\beta_i\frac{\ddpm_i-\sqrt{t_i}\mu}{t_i\sigma^2+1-t_i}}+\sqrt{h_i}Z_i
=\frac{g(t_{i+1})}{g(t_i)}\sqrt{t_i}\ddpm_i
+\frac{h_i}{g(t_i)}\mu+\sqrt{h_i}Z_i,
}
where $g(t)=(\sigma^2-1)t+1$. 
Therefore, noting that $\sqrt t_0\ddpm_0\sim N(t_0\mean,t_0I_d)$, we can show by induction that $\sqrt{t_i}\ddpm_i\sim N(\mu_i,\sigma_i^2I_d)$ for $i=0,1,\dots,N$, where $\mu_i\in\mathbb R^d$ and $\sigma_i^2>0$ satisfy the following recursion relations:
\ben{\label{eq:lb-rec}
\mu_{i+1}=\frac{g(t_{i+1})}{g(t_i)}\mu_i+\frac{h_i}{g(t_i)}\mu,
\qquad
\sigma_{i+1}^2=\bra{\frac{g(t_{i+1})}{g(t_i)}}^2\sigma_i^2+h_i.
}
Since $t_N=1$, Proposition 7 in \cite{GiSh84} gives
\[
\wass_2(\ddpm_N,\target)=\sqrt{|\mu_N-\mu|^2+d|\sigma_N-\sigma|^2}
\geq\min\{|\mu_N-\mu|,\,\sqrt d|\sigma_N-\sigma|\}.
\]
Therefore, it remains to prove
\ben{\label{eq:lb-aim}
|\mu_N-\mu|\geq\frac{\sigma^2}{\sigma^2\vee1}t_0|\mean-\mu|,\qquad
|\sigma_{N}-\sigma|\geq\frac{\sigma^4|\sigma^2-1|}{2(\sigma^5\vee1)}\bra{t_0^2+\frac{1-t_0^2}{2(\sigma^2\vee1)}\eta}.
}

First we consider $|\mu_N-\mu|$. 
By \eqref{eq:lb-rec}, for every $i=0,1,\dots,N-1$, 
\ba{
\frac{\mu_{i+1}}{g(t_{i+1})}=\frac{\mu_i}{g(t_i)}+\frac{h_i}{g(t_i)g(t_{i+1})}\mu.
}
Also, a simple computation shows
\ben{\label{eq:g-inv-diff}
\frac{1}{g(t_i)}-\frac{1}{g(t_{i+1})}=(\sigma^2-1)\frac{h_i}{g(t_i)g(t_{i+1})}.
}
Therefore,
\be{
\frac{\mu_{N}}{g(t_{N})}
=\frac{\mu_0}{g(t_0)}+\frac{1}{\sigma^2-1}\bra{\frac{1}{g(t_0)}-\frac{1}{g(t_{N})}}\mu
=\frac{\mu_0}{g(t_0)}+\frac{t_N-t_0}{g(t_0)g(t_N)}\mu.
}
Since $t_N=1$ and $\mu_0=t_0\mean$, we obtain
\[
\mu_N-\mu=\frac{\sigma^2t_0}{g(t_0)}\mean+\bra{\frac{1-t_0}{g(t_0)}-1}\mu
=\frac{\sigma^2t_0}{g(t_0)}(\mean-\mu).
\]
Since $g(t_0)\leq\sigma^2\vee1$, this gives the first inequality in \eqref{eq:lb-aim}.

Next we consider $|\sigma_N-\sigma|$. 
By \eqref{eq:lb-rec}, for every $i=0,1,\dots,N-1$, 
\ben{\label{eq:normal-var}
\frac{\sigma_{i+1}^2}{g(t_{i+1})^2}=\frac{\sigma_{i}^2}{g(t_{i})^2}+\frac{h_i}{g(t_{i+1})^2}.
}
Also, using \eqref{eq:g-inv-diff}, we obtain
\ba{
\frac{h_i}{g(t_{i+1})^2}=\frac{1}{\sigma^2-1}\bra{\frac{1}{g(t_i)}-\frac{1}{g(t_{i+1})}}-(\sigma^2-1)\frac{h_i^2}{g(t_i)g(t_{i+1})^2}.
}
Therefore,
\be{
\frac{\sigma_{N}^2}{g(t_N)^2}=\frac{t_0}{g(t_{0})^2}+\frac{t_N-t_0}{g(t_0)g(t_N)}-(\sigma^2-1)\sum_{i=0}^{N-1}\frac{h_i^2}{g(t_i)g(t_{i+1})^2}.
}
Since $t_N=1$ and
\ba{
\sigma^4\bra{\frac{t_0}{g(t_{0})^2}+\frac{1-t_0}{\sigma^2g(t_0)}}-\sigma^2=-\frac{\sigma^4(\sigma^2-1)t_0^2}{g(t_0)^2},
}
we obtain
\ba{
|\sigma_{N}^2-\sigma^2|=\sigma^4|\sigma^2-1|\bra{\frac{t_0^2}{g(t_0)^2}+\sum_{i=0}^{N-1}\frac{h_i^2}{g(t_i)g(t_{i+1})^2}}.
}
Since $h_i=t_{i+1}\beta_i\geq t_{i+1}\eta$ for all $i=0,1,\dots,N-1$ and $g(t_i)\leq\sigma^2\vee1$ for all $i=0,1,\dots,N$,
\ba{
\sum_{i=0}^{N-1}\frac{h_i^2}{g(t_i)g(t_{i+1})^2}
\geq\frac{\eta}{\sigma^6\vee1}\sum_{i=0}^{N-1}t_{i+1}h_i
\geq\frac{1-t_0^2}{2(\sigma^6\vee1)}\eta.
}
Consequently,
\ba{
|\sigma_{N}^2-\sigma^2|\geq\frac{\sigma^4|\sigma^2-1|}{\sigma^4\vee1}\bra{t_0^2+\frac{1-t_0^2}{2(\sigma^2\vee1)}\eta}.
}
In addition, by \eqref{eq:normal-var},
\ba{
\sigma_N^2=\sigma^2\bra{\frac{t_0}{g(t_0)^2}+\sum_{i=0}^{N-1}\frac{h_i}{g(t_{i+1})^2}}
\leq\frac{\sigma^2}{\sigma^2\wedge1}=\sigma^2\vee1.
}
Since $|\sigma_N-\sigma|=|\sigma_{N}^2-\sigma^2|/(\sigma_N+\sigma)$, we obtain the second inequality in \eqref{eq:lb-aim}.\qed

\subsection{Proof of Proposition \ref{prop:schedule}}\label{pr-prop:schedule}

(a) First, by Jordan's inequality, $t_0\geq 4^{-1}t_1\geq4^{-1}(\frac{2}{\pi}\eta_N)^2\geq16^{-1}\eta_N^2$. 
Next, since $\max_{i\in\indices}\beta_i\leq3/4$ is equivalent to $\max_{i\in\indices}t_{i+1}/t_i\leq4$, it remains to prove $t_{i+1}\leq4t_i$ for $i\in\indices$. 
If $i=0$, this holds by assumption. 
If $i>0$, we have 
$
\sqrt{t_{i+1}}=\sin((i+1)\eta_N)
=\sin(i\eta_N)\cos(\eta_N)+\cos(i\eta_N)\sin(\eta_N)
\leq2\sqrt{t_i}.
$
This completes the proof.

(b) Since 
$
\sin^2(t)
=\frac{1}{2}(1-\cos(2t))
$
for any $t\in\mathbb R$, we have for any $i\in\indices$
\besn{\label{eq:sum-to-prod}
h_i
&\leq-\frac{1}{2}\bra{\cos(2(i+1)\eta_N)-\cos(2i\eta_N)}
=\sin((2i+1)\eta_N)\sin(\eta_N)\\
&=2\sin((i+1/2)\eta_N)\cos((i+1/2)\eta_N)\sin(\eta_N).
}
Since
$
\sin((i+1/2)\eta_N)
\leq\sin((i+1)\eta_N)=\sqrt{t_{i+1}},
$
we conclude 
$
h_i\leq2\sqrt{t_{i+1}}\sin(\eta_N)
\leq2\eta_N\sqrt{t_{i+1}}.
$

(c) Since $s\geq(2N)^{-1}$, we have
\[
\sqrt{1-t_N}=\cos\bra{\frac{\pi}{2(1+s)}}=\sin\bra{\frac{\pi s}{2(1+s)}}\geq\sin(\eta_N/2).
\]
Therefore,
\ba{
\cos((i+1/2)\eta_N)
&=\cos((i+1)\eta_N)\cos(\eta_N/2)+\sin((i+1)\eta_N)\sin(\eta_N/2)\\
&\leq\sqrt{1-t_{i+1}}+\sqrt{1-t_N}\leq2\sqrt{1-t_{i+1}}.
} 
Combining this with \eqref{eq:sum-to-prod} gives
$
h_i\leq4\sqrt{1-t_{i+1}}\sin(\eta_N)\leq4\eta_N\sqrt{1-t_{i+1}}.
$
\qed

\subsection{Proof of Proposition \ref{prop:kl}}\label{pr-prop:kl}

The proof is a simplified version of that of \cite[Theorem 1]{JZL25}. 
The key idea is to construct a denoising procedure that generates the true target distribution via a ``noisy'' probability flow ordinary differential equation (ODE). 
Specifically, we use the construction based on the following ODE.
\begin{lemma}\label{lem:ode}
Let $0<t<t'<1$. Under the assumptions of \cref{prop:kl}, the following ODE admits a unique solution for any $y\in\mathbb R^d$:
\ben{\label{eq:ode}
\begin{cases}
\displaystyle\frac{\inte}{\inte s}T_{t,s}(y)=\frac{\nabla\log f_{s}(T_{t,s}(y))}{2s},\quad s\in(t,t'),\\
T_{t,t}(y)=y.
\end{cases}
}
Moreover, for any $s\in[t,t')$, the push-forward of $\slepian_t\target$ by $T_{t,s}$ coincides with $\slepian_s\target$.
\end{lemma}

\begin{proof}
By \eqref{eq:hess-lb} and \eqref{ass:kl}, $s^{-1}\|\nabla^2\log f_s(y)\|_{op}\leq L_0\{s(1-s)\}^{-1}$ for all $0<s<1$ and $y\in\mathbb R^d$. 
In particular, $\nabla\log f_s$ is Lipschitz continuous and has linear growth uniformly in $s\in[t,t']$; hence, by Theorem 16.2 in \cite{ABS24}, the ODE \eqref{eq:ode} has a unique solution for all $y\in\mathbb R^d$. 
Let $\nu_s$ be the push-forward  of $\slepian_t\target$ by $T_{t,s}$. 
Then, by Propositions 16.4 and 16.6 in \cite{ABS24}, $(\nu_s)_{s\in[t,t')}$ is the unique solution to the following equation in the distributional sense:
\be{
\begin{cases}
\nu_t=\slepian_t\target,\\
\frac{\inte}{\inte s}\nu_s+\diver\bra{\frac{\nabla\log f_{s}}{2s}\nu_s}=0.
\end{cases}
}
Therefore, it remains to show that $(\slepian_s\target)_{s\in[t,t')}$ solves the above equation. 
Since $\slepian_s\target$ has density $\pi_s$ for every $s$, it suffices to prove
\ben{\label{eq:ce}
\frac{\inte}{\inte s}\pi_s+\diver\bra{\frac{\nabla\log f_{s}}{2s}\pi_s}=0.
}
Using the fact that $(\pi_{e^{-2r}})_{r>0}$ is a solution to the Fokker--Planck equation associated with an OU process, we obtain
\ben{\label{eq:fp}
\frac{\inte}{\inte s}\pi_s(x)=-\frac{1}{2s}\bra{\diver(x\pi_s(x))+\Delta\pi_s(x)},
}
where $\Delta$ denotes the Laplacian. 
Since $\nabla\log f_s(x)=x+\pi_s(x)^{-1}\nabla\pi_s(x)$, we obtain \eqref{eq:ce}.
\end{proof}

We construct the denoising procedure that generates $\slepian_{t_N}\target$ as follows. 
For every $i=0,1,\dots,N-1$, set $\tau_i:=t_{i+1}(1-\beta_i)^{-1}=t_{i+1}^2/t_i$. 
Observe that
\ben{\label{tau-t}
\tau_i-t_i=\frac{t_{i+1}^2-t_i^2}{t_i}=\bra{\frac{t_{i+1}}{t_i}+1}h_i\leq3h_i
}
and
\ba{
1-\frac{1-\tau_i}{1-t_{i}}
=\frac{\tau_i-t_{i}}{1-t_{i}}
\leq\frac{3h_i}{1-t_{i}}
\leq3\eta,
}
where we used \eqref{eq:sampling2}. 
Since $\eta<1/6$, we have
\ben{\label{tau-well}
\frac{1-\tau_i}{1-t_{i}}>\frac{1}{2}.
}
In particular, $\tau_i<1$, and thus the map $T_{t_i,s}$ in \cref{lem:ode} is well-defined for $s\in[t_i,\tau_i]$. 
Using these maps, we define random vectors $\ddpmt_i$ $(i=0,1,\dots,N)$ by the following recurrence formula: $\ddpmt_0\sim\slepian_{t_0}\target$ and
\ba{
\ddpmt_{i+1}
&=\sqrt{1-\beta_i}T_{t_i,\tau_{i}}(\ddpmt_i)+\sqrt{\beta_i}Z_i.
}
Noting $(1-\beta_i)\tau_{i}=t_{i+1}$, we have $\ddpmt_{i+1}\sim\slepian_{t_{i+1}}\target$ by \cref{lem:ode}. 
In particular, $\kl(\slepian_{t_N}\target\mid P^{\ddpm_N})=\kl(P^{\ddpmt_N}\mid P^{\ddpm_N})$. 
Hence, by \cite[Theorem 2.16(b) and (c)]{PoWu25}, 
\be{
\kl(\slepian_{t_N}\target\mid P^{\ddpm_N})
\leq\kl(P^{\ddpmt_0}\mid P^{\ddpm_0})+\sum_{i=0}^{N-1}\int_{\mathbb R^d}\kl(\mcl L(\ddpmt_{i+1}\mid \ddpmt_i=y)\mid\mcl L(\ddpm_{i+1}\mid \ddpm_i=y))P^{\ddpmt_i}(\mathrm{d}y),
}
where $\mcl L(\xi_2\mid\xi_1=y)$ denotes the conditional law of $\xi_2$ given $\xi_1=y$ for two random vectors $\xi_1$ and $\xi_2$. 
For $i\in\indices$, integration by parts gives
\ba{
\frac{T_{t_i,\tau_i}(\ddpmt_i)}{\sqrt{\tau_i}}-\frac{\ddpmt_i}{\sqrt{t_i}}
&=\int_{t_i}^{\tau_i}\frac{1}{\sqrt s}\frac{\inte}{\inte s}T_{t_i,s}(\ddpmt_i)\inte s-\int_{t_i}^{\tau_i}\frac{T_{t_i,s}(\ddpmt_i)}{2s^{3/2}}\inte s
=\int_{t_i}^{\tau_i}\frac{\nabla\log \pi_{s}(T_{t_i,s}(\ddpmt_i))}{2s^{3/2}}\inte s.
}
Therefore,
\ba{
\ddpmt_{i+1}
&=\sqrt{(1-\beta_i)\tau_i}\bra{\frac{\ddpmt_i}{\sqrt{t_i}}+\int_{t_i}^{\tau_i}\frac{\nabla\log \pi_{s}(T_{t_i,s}(\ddpmt_i))}{2s^{3/2}}\inte s}+\sqrt{\beta_i}Z_i\\
&=\frac{\ddpmt_i}{\sqrt{1-\beta_i}}+\sqrt{t_{i+1}}\int_{t_i}^{\tau_i}\frac{\nabla\log \pi_{s}(T_{t_i,s}(\ddpmt_i))}{2s^{3/2}}\inte s+\sqrt{\beta_i}Z_i.
}
Since $\ddpmt_i$ is independent of $Z_i$, for any $y\in\mathbb R^d$,
\[
\mcl L(\ddpmt_{i+1}\mid \ddpmt_i=y)=N\bra{\frac{y}{\sqrt{1-\beta_i}}+\sqrt{t_{i+1}}\int_{t_i}^{\tau_i}\frac{\nabla\log \pi_{s}(T_{t_i,s}(y))}{2s^{3/2}}\inte s,\,\beta_iI_d}.
\]
Meanwhile, by definition,
\[
\mcl L(\ddpm_{i+1}\mid \ddpm_i=y)=N\bra{\frac{y}{\sqrt{1-\beta_i}}+\frac{\beta_i}{\sqrt{1-\beta_i}}\scored_{i}(y),\,\beta_iI_d}.
\]
Hence, Eq.(2.8) in \cite{PoWu25} gives
\ba{
&\kl(\mcl L(\ddpmt_{i+1}\mid \ddpmt_i=y)\mid\mcl L(\ddpm_{i+1}\mid \ddpm_i=y))\\
&=\frac{1}{2\beta_i}\abs{\frac{\beta_i}{\sqrt{1-\beta_i}} \scored_{i}(y)-\sqrt{t_{i+1}}\int_{t_i}^{\tau_i}\frac{\nabla\log \pi_{s}(T_{t_i,s}(y))}{2s^{3/2}}\inte s}^2.
}
Consequently, 
\besn{\label{eq:kl-chain}
&\kl(\slepian_{t_N}\target\mid P^{\ddpm_N})\\
&\leq\kl(P^{\ddpmt_0}\mid P^{\ddpm_0})+
\sum_{i=0}^{N-1}\frac{1}{2\beta_i}\norm{\frac{\beta_i}{\sqrt{1-\beta_i}}\scored_{i}(\ddpmt_i)-\sqrt{t_{i+1}}\int_{t_i}^{\tau_i}\frac{\nabla\log \pi_{s}(T_{t_i,s}(\ddpmt_i))}{2s^{3/2}}\inte s}_2^2.
}
The first term on the right-hand side is bounded as follows.
\begin{lemma}\label{lem:kl-init}
We have
\[
\kl(P^{\ddpmt_0}\mid P^{\ddpm_0})\leq \frac{t_0}{2}\bra{|\mu-\mean|^2+\trace(\Sigma)+d\frac{t_0}{2(1-t_0)^2}}.
\]
\end{lemma}

\begin{proof}
We may write $\ddpmt_0=\sqrt{t_0}\xi+\sqrt{1-t_0}Z$ and $\ddpm_0=\sqrt{t_0}\mean+Z$, where $\xi\sim\target$ and $Z\sim N(0,I_d)$ are independent. 
Then, conditional on $\xi$, we have $\ddpmt_0\sim N(\sqrt{t_0}\xi,(1-t_0)I_d)$ and $\ddpm_0\sim N(\sqrt{t_0}\mean,I_d)$. 
Hence, by Theorem 2.15 and Eq.(2.8) in \cite{PoWu25},
\ba{
\kl(P^{\ddpmt_0}\mid P^{\ddpm_0})
&\leq\frac{1}{2}\E\sbra{t_0|\xi-\mean|^2-d\log(1-t_0)-dt_0}\\
&=\frac{1}{2}\E\sbra{t_0|\xi-\mean|^2+dt_0^2\int_0^1\frac{1-\theta}{(1-t_0\theta)^2}\inte\theta}
\leq\frac{t_0}{2}\bra{|\mu-\mean|^2+\trace(\Sigma)+d\frac{t_0}{2(1-t_0)^2}}.
}
This completes the proof.
%
\end{proof}

It remains to bound the second term on the right-hand side of \eqref{eq:kl-chain}. 
We begin by an auxiliary lemma.
\begin{lemma}\label{lemma:yu}
Let $Y$ be a centered random vectors in $\mathbb R^d$ such that $\E[|Y|^2]<\infty$. Also, let $U$ be a square-integrable random variable. Then, $|\E[UY]|^2\leq\Var[U]\|\Cov[Y]\|_{op}.$
\end{lemma}

\begin{proof}
Since $|\E[UY]|=\sup_{a\in\mathbb R^d:|a|=1}a\cdot\E[UY]$, it suffices to prove $|a\cdot\E[UY]|^2\leq\Var[U]\|\Cov[Y]\|_{op}$ for every unit vector $a\in\mathbb R^d$. 
Since $\E[Y]=0$, we have $a\cdot\E[UY]=a\cdot\E[(U-\E[U])Y]$. 
Hence, by the Schwarz inequality,
$
|a\cdot\E[UY]|^2\leq\Var[U]\E[(a\cdot Y)^2]=\Var[U]\langle\Cov[Y],\,a^{\otimes2}\rangle
\leq\Var[U]\|\Cov[Y]\|_{op}.
$
\end{proof}

The following lemma bounds the discretization error of the second term on the right-hand side of \eqref{eq:kl-chain}.
\begin{lemma}\label{lem:kl-main-0}
Under the assumptions of \cref{prop:kl}, for any $i=0,1,\dots,N-1$ and $s\in[t_i,\tau_i]$,
\ben{\label{eq:kl-main-0}
\|\nabla\log \pi_{s}(T_{t_i,s}(\ddpmt_i))-\nabla\log \pi_{t_i}(\ddpmt_i)\|_2
\leq\frac{36\sqrt2L_0\sqrt d}{t_i(1-t_i)^{3/2}}h_i.
}
\end{lemma}

\begin{proof}
For any $0<s<1$ and $y\in\mathbb R^d$, set $F(s,y)=\nabla\log \pi_{s}(y)$. 
By the fundamental theorem of calculus,
\ben{\label{kl-ftc}
\nabla\log \pi_{\tau}(T_{t_i,\tau}(\ddpmt_i))-\nabla\log \pi_{t_i}(\ddpmt_i)
=\int_{t_i}^\tau\bra{\frac{\partial F}{\partial s}(s,T_{t_i,s}(\ddpmt_i))+\nabla F(s,T_{t_i,s}(\ddpmt_i))\frac{\inte}{\inte s}T_{t_i,s}(\ddpmt_i)}\inte s.
}
For any $s\in[t_i,\tau]$ and $y\in\mathbb R^d$, we have
\ba{
\frac{\partial F}{\partial s}(s,y)
=\frac{1}{\pi_s(y)}\frac{\inte}{\inte s}\nabla\pi_s(y)-\frac{\nabla\pi_s(y)}{\pi_s(y)^2}\frac{\inte}{\inte s}\pi_s(y)
=\frac{1}{\pi_s(y)}\nabla\bra{\frac{\inte}{\inte s}\pi_s}(y)-\frac{\nabla\pi_s(y)}{\pi_s(y)^2}\frac{\inte}{\inte s}\pi_s(y).
}
Inserting \eqref{eq:fp} into this identity, we obtain
\ba{
\frac{\partial F}{\partial s}(s,y)
&=-\frac{1}{2s}\bra{d\frac{\nabla\pi_s(y)}{\pi_s(y)}+\frac{\nabla\pi_s(y)}{\pi_s(y)}+\frac{\nabla^2\pi_s(y)}{\pi_s(y)}y+\frac{\nabla(\Delta\pi_s)(y)}{\pi_s(y)}}\\
&\quad+\frac{\nabla\pi_s(y)}{2s\pi_s(y)}\bra{d+y\cdot\frac{\nabla\pi_s(y)}{\pi_s(y)}+\frac{\Delta\pi_s(y)}{\pi_s(y)}}\\
&=-\frac{1}{2s}\bra{\frac{\nabla\pi_s(y)}{\pi_s(y)}+\frac{\nabla^2\pi_s(y)}{\pi_s(y)}y+\frac{(\nabla\Delta\pi_s)(y)}{\pi_s(y)}}
+\frac{\nabla\pi_s(y)^{\otimes2}}{2s\pi_s(y)^2}y+\frac{\nabla\pi_s(y)}{2s\pi_s(y)}\frac{\Delta\pi_s(y)}{\pi_s(y)}.
}
Meanwhile, by \eqref{eq:ode},
\ba{
\nabla F(s,y)\frac{\inte}{\inte s}T_{t_i,s}(y)
&=\nabla^2\log\pi_s(y)\frac{\nabla\log f_s(y)}{2s}\\
&=\bra{\frac{\nabla^2\pi_s(y)}{\pi_s(y)}-\frac{\nabla\pi_s(y)^{\otimes2}}{\pi_s(y)^2}}\frac{y}{2s}
+\nabla^2\log\pi_s(y)\frac{\nabla\pi_s(y)}{2s\pi_s(y)}.
}
Consequently,
\besn{\label{eq:kl-main-integrand}
&\frac{\partial F}{\partial s}(s,y)+\nabla F(s,y)\frac{\inte}{\inte s}T_{t_i,s}(y)\\
&=\frac{1}{2s}\bra{\nabla^2\log\pi_s(y)-I_d}\frac{\nabla\pi_s(y)}{\pi_s(y)}
+\frac{1}{2s}\bra{\frac{\nabla\pi_s(y)}{\pi_s(y)}\frac{\Delta\pi_s(y)}{\pi_s(y)}-\frac{\nabla(\Delta\pi_s)(y)}{\pi_s(y)}}.
}
To obtain a more tractable expression, we define a probability distribution $Q_{s,y}$ on $\mathbb R^d$ by
\[
Q_{s,y}(\mathrm{d}x)=\frac{1}{\pi_s(y)(1-s)^{d/2}}\phi_d\bra{\frac{y-\sqrt sx}{\sqrt{1-s}}}\target(\mathrm{d}x).
\]
A straightforward computation shows
\ba{
\frac{\nabla\pi_s(y)}{\pi_s(y)}&=-\frac{1}{1-s}\int_{\mathbb R^d}(y-\sqrt sx)Q_{s,y}(\mathrm{d}x),\\
\frac{\nabla^2\pi_s(y)}{\pi_s(y)}&=-\frac{I_d}{1-s}+\frac{1}{(1-s)^2}\int_{\mathbb R^d}(y-\sqrt sx)^{\otimes2}Q_{s,y}(\mathrm{d}x),\\
\frac{\Delta\pi_s(y)}{\pi_s(y)}&=-\frac{d}{1-s}+\frac{1}{(1-s)^2}\int_{\mathbb R^d}|y-\sqrt sx|^2Q_{s,y}(\mathrm{d}x),\\
\frac{\nabla(\Delta\pi_s)(y)}{\pi_s(y)}&=-\frac{d+2}{1-s}\frac{\nabla\pi_s(y)}{\pi_s(y)}-\frac{1}{(1-s)^3}\int_{\mathbb R^d}|y-\sqrt sx|^2(y-\sqrt sx)Q_{s,y}(\mathrm{d}x).
}
Now, let $\xi\sim\target$ and $Z\sim N(0,I_d)$ be independent, and set $\xi_s:=\sqrt s\xi+\sqrt{1-s}Z$. 
Then, we can easily verify that $Q_{s,y}$ is the conditional law of $\xi$ given $\xi_s=y$. 
Hence,
\ben{\label{eq:tweedie}
\frac{\nabla\pi_s(\xi_s)}{\pi_s(\xi_s)}=-\frac{\E[Z\mid\xi_s]}{\sqrt{1-s}},\qquad
\nabla^2\log\pi_s(\xi_s)=-\frac{I_d}{1-s}+\frac{\Cov[Z\mid\xi_s]}{1-s}
}
and
\ben{\label{eq:pi-laplace}
\frac{\nabla\pi_s(\xi_s)}{\pi_s(\xi_s)}\frac{\Delta\pi_s(\xi_s)}{\pi_s(\xi_s)}-\frac{\nabla(\Delta\pi_s)(\xi_s)}{\pi_s(\xi_s)}=\frac{2\E[Z\mid\xi_s]}{(1-s)^{3/2}}
-\frac{\E[|Z|^2(Z-\E[Z\mid\xi_s])\mid\xi_s]}{(1-s)^{3/2}}.
}
By the second identity in \eqref{eq:tweedie}, we have $\nabla^2\log\pi_s(\xi_s)\succeq-(1-s)^{-1}I_d$. Combining this with \eqref{ass:kl} gives $\|\nabla^2\log\pi_s(\xi_s)\|_{op}\leq L_0(1-s)^{-1}$. 
Then, noting that $T_{t_i,s}(\ddpmt_i)\overset{d}{=}\xi_{s}$, we obtain by \eqref{eq:kl-main-integrand}, \eqref{eq:tweedie} and \eqref{eq:pi-laplace}
\ban{
&\norm{\frac{\partial F}{\partial s}(s,T_{t_i,s}(\ddpmt_i))+\nabla F(s,T_{t_i,s}(\ddpmt_i))\frac{\inte}{\inte s}T_{t_i,s}(\ddpmt_i)}_2\notag\\
&\leq\frac{1}{2s}\bra{\frac{L_0}{1-s}+1}\frac{\|\E[Z\mid\xi_s]\|_2}{\sqrt{1-s}}
+\frac{1}{2s}\bra{\frac{2\|\E[Z\mid\xi_s]\|_2}{(1-s)^{3/2}}
+\frac{\|\E[|Z|^2(Z-\E[Z\mid\xi_s])\mid\xi_s]\|_2}{(1-s)^{3/2}}}\notag\\
&\leq\frac{2L_0\sqrt d}{s(1-s)^{3/2}}+\frac{\|\E[|Z|^2(Z-\E[Z\mid\xi_s])\mid\xi_s]\|_2}{s(1-s)^{3/2}}.\label{eq:kl-main-est}
}
Applying \cref{lemma:yu} conditional on $\xi_s$ gives
\ba{
\abs{\E[|Z|^2(Z-\E[Z\mid\xi_s])\mid\xi_s]}^2
\leq\Var[|Z|^2\mid\xi_s]\|\Cov[Z\mid\xi_s]\|_{op}.
}
Since $\|\Cov[Z\mid\xi_s]\|_{op}\leq L_0+1\leq2L_0$ by \eqref{ass:kl} and the second identity in \eqref{eq:tweedie}, we obtain
\ba{
\norm{\E[|Z|^2(Z-\E[Z\mid\xi_s])\mid\xi_s]}_2^2
\leq2L_0\E[\Var[|Z|^2\mid\xi_s]]
\leq2L_0\Var[|Z|^2]
=4L_0d.
}
Inserting this into \eqref{eq:kl-main-est} gives
\ba{
\norm{\frac{\partial F}{\partial s}(s,T_{t_i,s}(\ddpmt_i))+\nabla F(s,T_{t_i,s}(\ddpmt_i))\frac{\inte}{\inte s}T_{t_i,s}(\ddpmt_i)}_2
\leq\frac{2L_0\sqrt d}{s(1-s)^{3/2}}+\frac{4\sqrt{L_0d}}{s(1-s)^{3/2}}
\leq\frac{6L_0\sqrt d}{s(1-s)^{3/2}}.
}
Therefore, by \eqref{kl-ftc} and the integral Minkowski inequality, we conclude
\ba{
\|\nabla\log \pi_{s}(T_{t_i,s}(\ddpmt_i))-\nabla\log \pi_{t_i}(\ddpmt_i)\|_2
&\leq\int_{t_i}^{\tau_i}\frac{6L_0\sqrt d}{s(1-s)^{3/2}}\inte s
\leq\frac{18L_0\sqrt d}{t_i(1-\tau_i)^{3/2}}h_i,
}
where we used \eqref{tau-t} for the last inequality. 
Applying \eqref{tau-well} to the last bound gives the desired result. 
\end{proof}

\begin{lemma}\label{lem:kl-main}
Under the assumptions of \cref{prop:kl},
\[
\sum_{i=0}^{N-1}\frac{1}{2\beta_i}\norm{\frac{\beta_i}{\sqrt{1-\beta_i}}\nabla\log \pi_{t_i}(\ddpmt_i)-\sqrt{t_{i+1}}\int_{t_i}^{\tau_i}\frac{\nabla\log \pi_{s}(T_{t_i,s}(\ddpmt_i))}{2s^{3/2}}\inte s}_2^2
\lesssim L_0^2d\eta^2\log\frac{1}{t_0\delta}.
\]
\end{lemma}

\begin{proof}
For $i=0,1,\dots,N-1$, observe that
\ba{
\sqrt{t_{i+1}}\int_{t_i}^{\tau_{i}}\frac{\inte s}{2s^{3/2}}=\sqrt{t_{i+1}}\bra{\frac{1}{\sqrt{t_i}}-\frac{1}{\sqrt{\tau_i}}}
=\frac{1}{\sqrt{1-\beta_i}}-\sqrt{1-\beta_i}=\frac{\beta_i}{\sqrt{1-\beta_i}}.
}
Combining this identity with \cref{lem:kl-main-0} and the integral Minkowski inequality gives
\ba{
&\norm{\frac{\beta_i}{\sqrt{1-\beta_i}}\nabla\log \pi_{t_i}(\ddpmt_i)-\sqrt{t_{i+1}}\int_{t_i}^{\tau_i}\frac{\nabla\log \pi_{s}(T_{t_i,s}(\ddpmt_i))}{2s^{3/2}}\inte s}_2\\
&=\norm{\sqrt{t_{i+1}}\int_{t_i}^{\tau_i}\frac{\nabla\log \pi_{t_i}(\ddpmt_i)-\nabla\log \pi_{s}(T_{t_i,s}(\ddpmt_i))}{2s^{3/2}}\inte s}_2\\
&\lesssim\sqrt{t_{i+1}}\int_{t_i}^{\tau_{i}}\frac{\inte s}{2s^{3/2}}\frac{L_0\sqrt d}{t_i(1-t_i)^{3/2}}h_i
=\frac{\beta_i}{\sqrt{1-\beta_i}}\frac{L_0\sqrt d}{t_i(1-t_i)^{3/2}}h_i
\leq\frac{\beta_i}{\sqrt{1-\beta_i}}\frac{L_0\sqrt d}{\sqrt{1-t_i}}2\eta,
}
where the last inequality follows from \eqref{eq:sampling2}. 
Consequently,
\ba{
&\sum_{i=0}^{N-1}\frac{1}{2\beta_i}\norm{\frac{\beta_i}{\sqrt{1-\beta_i}}\nabla\log \pi_{t_i}(\ddpmt_i)-\sqrt{t_{i+1}}\int_{t_i}^{\tau_i}\frac{\nabla\log \pi_{s}(T_{t_i,s}(\ddpmt_i))}{2s^{3/2}}\inte s}_2^2\\
&\lesssim L_0^2d\eta^2\sum_{i=0}^{N-1}\frac{\beta_i}{(1-\beta_i)(1-t_i)}
=L_0^2d\eta^2\sum_{i=0}^{N-1}\frac{h_i}{t_i(1-t_i)}
\lesssim L_0^2d\eta^2\log\frac{1}{t_0\delta}.
}
This completes the proof.
\end{proof}

\begin{proof}[\bf\upshape Proof of \cref{prop:kl}]
By \eqref{score-error-p} and \eqref{eq:sampling2},
\ba{
\sum_{i=0}^{N-1}\frac{1}{2\beta_i}\norm{\frac{\beta_i}{\sqrt{1-\beta_i}}\{\scored_i(\ddpmt_i)-\nabla\log \pi_{t_i}(\ddpmt_i)\}}_2^2
&\leq\sum_{i=0}^{N-1}\frac{\beta_i}{2(1-\beta_i)}\eps_\mathrm{score}(t_i)^2
=\sum_{i=0}^{N-1}\frac{h_i}{2t_i}\eps_\mathrm{score}(t_i)^2\\
&\leq\sum_{i=0}^{N-1}\frac{h_i}{t_{i+1}}\eps_\mathrm{score}(t_i)^2
\leq\int_{t_0}^{t_N}\frac{\eps_\mathrm{score}(s)^2}{s}\inte s.
}
Combining this bound with \eqref{eq:kl-chain}, \cref{lem:kl-init} and \cref{lem:kl-main} gives the desired result.
\end{proof}

\subsection{Proof of Lemma \ref{lem:lp-o}}\label{sec:lp-o}

As pointed out in \cite{RiXu16}, an inequality of the form \eqref{eq:lp-o} is a consequence of the 2-smoothness of the underlying Banach space. 
In our setting, the latter corresponds to the following result.
\begin{lemma}\label{2-smooth}
Let $X$ and $Y$ be two random vectors in $\mathbb R^d$. 
Then, for any $p\geq2$,
\ba{
\|X+Y\|_p^2+\|X-Y\|_p^2\leq2\|X\|_p^2+2(p-1)\|Y\|_p^2.
}
\end{lemma}

\begin{proof}
This inequality follows as a special case of \cite[Proposition 2.2]{vNV22}; see also \cite[Theorem 1]{BCL94}.
\end{proof}

\begin{proof}[\bf\upshape Proof of Lemma \ref{lem:lp-o}]
We follow the proof of \cite[Lemma A.1]{HNWTW22}. 
Without loss of generality, we may assume $\|X\|_p+\|Y\|_p<\infty$. 
Fix $n\in\mathbb N$ arbitrarily, and set $Z:=n^{-1}Y$. 
Also, let $D_k:=\|X+kZ\|_p^2-\|X+(k-1)Z\|_p^2-2k(p-1)\|Z\|_p^2$ for every integer $k$. 
Since \cref{2-smooth} gives
\[
\|X+kZ\|_p^2+\|X+(k-2)Z\|_p^2\leq2\|X+(k-1)Z\|_p^2+2(p-1)\|Z\|_p^2,
\]
we have $D_k\leq D_{k-1}$. 
Meanwhile, since $\E[Z\mid X]=n^{-1}\E[Y\mid X]=0$, 
\ba{
\|X+Z\|_p^2+\|X\|_p^2
&=\|X+Z\|_p^2+\|\E[X-Z\mid X]\|_p^2\\
&\leq\|X+Z\|_p^2+\|X-Z\|_p^2
\leq2\|X\|_p^2+2(p-1)\|Z\|_p^2,
}
where the first inequality is by Jensen's inequality and the second by \cref{2-smooth}. Rearranging the terms gives $D_1\leq0$. 
This yields $D_k\leq0$ for all $k\geq1$. 
Using this inequality, we obtain
\ba{
\|X+Y\|_p^2-\|X\|_p^2
&=\sum_{k=1}^n\bra{D_k+2k(p-1)\|Z\|_p^2}\\
&\leq n(n+1)(p-1)\|Z\|_p^2=\bra{1+\frac{1}{n}}(p-1)\|Y\|_p^2.
}
Letting $n\to\infty$ gives the desired result. 
\end{proof}

\subsection{Proof of Lemma \ref{V-bound}}\label{sec:V-bound}

We first prove \eqref{v-op-bound} following the argument in \cite[Section 7]{KlLe25ams}. 
The next lemma is a variant of \cite[Lemma 7.3]{KlLe25ams} and follows from it via a time change (see also \cite[Section 4]{EMZ20}). We give a direct proof.
\begin{lemma}\label{lem:v-sde}
For any $0<t<1$,
\ben{\label{eq:mv-rep}
m_t=\mu+\int_0^t\Gamma_s\inte W_s,\qquad
V_t=\Sigma+\sum_{i=1}^d\int_0^tH_{i,s}\inte W^i_{s}-\int_0^t\Gamma_s^2\inte s,
}
where
\[
H_{i,s}=\frac{1}{1-s}\int_{\mathbb R^d}\bra{y_i-m^i_{s}}\bra{y-m_s}^{\otimes2}\target_{s,X_s}(\mathrm{d}y).
\]
\end{lemma}

\begin{proof}
The first identity in \eqref{eq:mv-rep} follows from \cite[Proposition 2.7]{Ko26ddpm}. 
To prove the second identity, define a function $F:(0,1)\times\mathbb R^d\to\mathbb R^{d}\otimes\mathbb R^d$ by
\[
F(t,x)=\int_{\mathbb R^d}y^{\otimes2}\target_{t,x}(\mathrm{d}y),\quad 0<t<1,~x\in\mathbb R^d.
\]
By \cref{sl-deriv}, for any $0<s<1$ and $i=1,\dots,d$,
\[
\frac{\partial F}{\partial x_i}(s,x)
=\frac{1}{1-s}\int_{\mathbb R^d}y^{\otimes2}(y_i-m_i(s,x))\target_{s,x}(\mathrm{d}y),
\]
where $m_i(s,x)$ denotes the $i$-th component of $m(s,x)$. 
Hence, using the Markov property of $X$,
\[
\frac{\partial F}{\partial x_i}(s,X_s)
=\frac{\E[X_1^{\otimes2}(X_1^i-\E[X^i_1\mid \mcl F_s])\mid X_s]}{1-s}
\to\E[X_1^{\otimes2}(X_1^i-\E[X^i_1])]\quad\text{a.s. }~(s\downarrow0),
\]
where the last convergence follows from \cite[Corollary 2.7.8]{KaSh98} and \cite[Corollary II.2.4]{ReYo99}. 
Also, $F(s,X_s)=\E[X_1^{\otimes2}\mid X_s]=\E[X_1^{\otimes2}\mid \mcl F_s]$. 
Therefore, a similar argument to the proof of \cref{lem:v-rep} yields
\ba{
F(t,X_t)=\E[X_1^{\otimes2}]+
\sum_{i=1}^d\int_0^{t}\frac{\partial F}{\partial x_i}(s,X_s)dW^i_s.
}
Meanwhile, by the first identity in \eqref{eq:mv-rep} and integration by parts give
\ba{
m_t^{\otimes2}=\mu^{\otimes2}+\sum_{i=1}^d\int_0^tm_s\otimes\Gamma^{\cdot i}_s\inte W^i_s
+\sum_{i=1}^d\int_0^t\Gamma^{\cdot i}_s\otimes m_s\inte W^i_s
+\int_0^t\Gamma_s^2\inte s,
}
where $\Gamma^{\cdot i}_s$ denotes the $i$-th column vector of $\Gamma_s$. 
Since 
\ba{
&\frac{\partial F}{\partial x_i}(s,X_s)-m_s\otimes\Gamma^{\cdot i}_s-\Gamma^{\cdot i}_s\otimes m_s\\
&=\frac{1}{1-s}\int_{\mathbb R^d}\cbra{y^{\otimes2}(y_i-m^i_s)-m_s\otimes(y-m_s)(y_i-m^i_s)-(y-m_s)\otimes m_s(y_i-m^i_s)}\target_{s,X_s}(\mathrm{d}y)\\
&=H_{i,s},
}
we obtain the second identity in \eqref{eq:mv-rep} because $V_t=F(t,X_t)-m_t^{\otimes2}$.
\end{proof}

The following lemma is an extension of \cite[Theorem 7.10]{KlLe25ams} to the general covariance matrix case.
\begin{lemma}\label{Gamma-tail}
Suppose that $\target$ is log-concave and $\|\Sigma\|_{op}>0$. 
Then, there exists a universal constant $C>0$ such that for any $0<t\leq\min\{10^{-1},\,(2^{15}\|\Sigma\|_{op}\log^2(d+1))^{-1}\}$, 
\ba{
P(\|V_r\|_{op}\geq2\|\Sigma\|_{op}\text{ for some }r\in[0,t])
\leq \exp\bra{-\frac{1}{C\|\Sigma\|_{op}t}}.
}
\end{lemma}

\begin{proof}
Observe that $\target_{t,X_t}$ is $\frac{t}{1-t}$-strongly log-concave for any $0<t<1$. 
Then, by a similar argument to the proof of \cite[Eq.(7.2)]{KlLe25ams} with using \cref{lem:v-sde} instead of \cite[Lemma 7.3]{KlLe25ams}, we obtain for any $\beta>0$
\[
\|V_t\|_{op}\leq \|\Sigma\|_{op}+\frac{\log d}{\beta}+Z_t+2\beta\int_0^t\frac{\|V_s\|^{5/2}_{op}}{\sqrt s(1-s)^{3/2}}\inte s,
\]
where $Z=(Z_t)_{t\in[0,1)}$ is a continuous local martingale (depending on $\beta$) such that $Z_0=0$ and 
\ben{\label{z-pqv}
\langle Z,Z\rangle_r\leq C_1\int_0^r\frac{\|V_s\|_{op}^3}{(1-s)^2}\inte s
}
for all $0\leq r<1$, where $C_1$ is a universal constant. 
Letting $\beta=2\|\Sigma\|_{op}^{-1}\log d$, we obtain
\[
\|V_t\|_{op}\leq \frac{3}{2}\|\Sigma\|_{op}+Z_t+\frac{4\log d}{\|\Sigma\|_{op}}\int_0^t\frac{\|V_s\|^{5/2}_{op}}{\sqrt s(1-s)^{3/2}}\inte s.
\]
Now, fix $0<t\leq\min\{10^{-1},\,(2^{15}\|\Sigma\|_{op}\log^2(d+1))^{-1}\}$. 
Assume $\|V_r\|_{op}\geq2\|\Sigma\|_{op}$ for some $r\in[0,t]$, and let $r$ be the smallest such time. 
Then, $\|V_s\|_{op}\leq2\|\Sigma\|_{op}$ for all $s\in[0,r]$. 
Since $(1-s)^{-3/2}\leq\sqrt 2$ for any $s\leq10^{-1}$,
\ba{
2\|\Sigma\|_{op}
&\leq\|V_r\|_{op}
\leq \frac{3}{2}\|\Sigma\|_{op}+Z_r+\frac{4\log d}{\|\Sigma\|_{op}}\cdot2^{5/2}\|\Sigma\|_{op}^{5/2}\int_0^r\frac{1}{\sqrt{s}(1-s)^{3/2}}\inte s\\
&\leq \frac{3}{2}\|\Sigma\|_{op}+Z_r+(2^6\|\Sigma\|_{op}^{3/2}\log d)\sqrt r.
}
Since $t\leq(2^{15}\|\Sigma\|_{op}\log^2(d+1))^{-1}$, we obtain $Z_r\geq(\frac{1}{2}-\frac{1}{2\sqrt 2})\|\Sigma\|_{op}\geq\frac{1}{8}\|\Sigma\|_{op}$.  Also, \eqref{z-pqv} yields $\langle Z,Z\rangle_r\leq16C_1\|\Sigma\|_{op}^3r$. Hence
\ba{
&P(\|V_r\|_{op}\geq2\|\Sigma\|_{op}\text{ for some }r\in[0,t])\\
&\leq P(Z_r\geq \|\Sigma\|_{op}/8\text{ and }\langle Z,Z\rangle_r\leq16C_1\|\Sigma\|_{op}^3r\text{ for some }r\in[0,t])
\leq \exp\bra{-\frac{1}{512C_1\|\Sigma\|_{op}t}},
}
where the last inequality is by Freedman's inequality \cite[Lemma 7.9]{KlLe25ams}.
\end{proof}

\eqref{v-op-bound} immediately follows from \cref{Gamma-tail}:
\begin{proof}[\bf\upshape Proof of \eqref{v-op-bound}]
First, if $\|\Sigma\|_{op}=0$, then $\target$ is the unit mass at a point; hence, $\target_{t,x}=\target$ for all $x\in\mathbb R^d$. This implies $V_t=\Sigma=0$, so \eqref{v-op-bound} holds. 

Next, assume $\|\Sigma\|_{op}>0$. 
Since $\target_{t,X_t}$ is $\frac{t}{1-t}$-strongly log-concave, we have 
\ben{\label{v-bl-bound}
\|V_t\|_{op}\leq\frac{1-t}{t}
} 
by the Brascamp--Lieb inequality. 
Combining this fact with \cref{Gamma-tail} gives
\ba{
\E[\|V_t\|_{op}^p]\leq(2\|\Sigma\|_{op})^p+\bra{\frac{1-t}{t}}^p\exp\bra{-\frac{1}{C\|\Sigma\|_{op}t}}
\leq(2\|\Sigma\|_{op})^p+p!\bra{C\|\Sigma\|_{op}}^p,
}
where $C>0$ is a universal constant. This shows \eqref{v-op-bound}. 
\end{proof}

Next, we prove \eqref{v-fr-bound} following the argument in \cite[Section 5]{KlLe26}. 
We need to introduce additional notation. 
For $0<t<1$ and $x\in\mathbb R^d$, we define a probability distribution $\tilde{P}^*_{t,x}$ on $\mathbb R^d$ by $\tilde{P}^*_{t,x}(B)=\target_{t,x}(\sqrt{\Lambda}B)$ for every Borel set $B\subset\mathbb R^d$. 
Note that $\tilde{P}^*_{t,x}$ is the law of $\Lambda^{-1/2}\xi$ with $\xi\sim\target_{t,x}$. 
Then, we set $\tilde{P}^*_t:=\tilde{P}^*_{t,X_t}$. 
We write $\tilde V_t$ for the covariance matrix of $\tilde{P}^*_t$, i.e., $\tilde V_t:=\Lambda^{-1}V_t$. 
We denote by $\lambda_1(t)\geq\cdots\geq\lambda_d(t)$ the eigenvalues of $\tilde V_t$. 
Further, let $u_1(t),\dots,u_d(t)\in\mathbb R^d$ be an orthonormal basis of eigenvectors corresponding to $\lambda_1(t),\dots,\lambda_d(t)$. 
For $i,j,k\in\{1,\dots,d\}$, define
\[
\xi_{ij}(t)=\int_{\mathbb R^d}\langle x-m_{t},\, u_i(t)\rangle\langle x-m_{t},\, u_j(t)\rangle(x-m_t)\tilde{P}^*_{t}(\mathrm{d}x)\in\mathbb R^d.
\]
Finally, for a function $f:\mathbb R\to\mathbb R$ and a $d\times d$ symmetric matrix $A$ with spectral decomposition $A=\sum_{i=1}^d\lambda_iu_i\otimes u_i$, we set $f(A)=\sum_{i=1}^df(\lambda_i)u_i\otimes u_i$.


The following lemma is a variant of \cite[Lemma 5.3]{KlLe26}.
\begin{lemma}\label{kl26-lem5.3}
Let $f:\mathbb R\to\mathbb R$ be a $C^2$ function of polynomial growth. Then, for any $0<t<1$,
\ba{
\frac{\inte}{\inte t}\E[\trace f(\tilde V_t)]
=\frac{1}{(1-t)^2}\bra{\frac{\Lambda}{2}\sum_{i,j=1}^d\E\sbra{|\xi_{ij}(t)|^2\frac{f'(\lambda_i(t))-f'(\lambda_j(t))}{\lambda_i(t)-\lambda_j(t)}}
-\frac{1}{\Lambda}\sum_{i=1}^d\E[\lambda_i(t)^2f'(\lambda_i(t))]},
}
where we interpret the quotient as $f''(\lambda_i(t))$ when $\lambda_i(t)=\lambda_j(t)$. 
\end{lemma}

\begin{proof}
For any real-valued $C^2$ function $F$ on the space of $d\times d$ symmetric matrices, \cref{lem:v-sde} and It\^o's formula gives
\ba{
F(\tilde V_t)&=F(\Lambda^{-1}\Sigma)+\frac{1}{\Lambda}\sum_{i=1}^d\int_0^t\trace\bra{\nabla F(\tilde V_s)H_{i,s}}\inte W^i_{s}\\
&\quad-\frac{1}{\Lambda}\int_0^t\trace\bra{\nabla F(\tilde V_s)\Gamma_s^2}\inte s
+\frac{1}{2\Lambda^2}\sum_{i=1}^d\int_0^t\langle \nabla^2F(\tilde V_s),\, H_{i,s}^{\otimes2}\rangle\inte s,
}
where 
\[
\nabla F(A)=\bra{\frac{\partial F}{\partial A_{ij}}(A)}_{1\leq i,j\leq d}\in(\mathbb R^{d})^{\otimes2},\quad
\nabla^2 F(A)=\bra{\frac{\partial^2 F}{\partial A_{ij}\partial A_{kl}}(A)}_{1\leq i,j,k,l\leq d}\in(\mathbb R^{d})^{\otimes4}.
\]
By \cite[Lemma 7.6]{KlLe25ams},
\[
\|H_{i,s}\|_{op}\lesssim\frac{\|V_s\|_{op}^{3/2}}{1-s}.
\]
Also, note that
\[
\norm{\|V_r\|_{op}}_p\lesssim \Lambda\bra{p+\log^2(d+1)}
\]
for any $p\geq2$ by \eqref{v-op-bound} and \eqref{v-bl-bound}. 
Therefore, if there exist constants $C,q>0$ such that
\[
\abs{\frac{\partial F}{\partial A_{ij}}(A)}\leq C(1+\|A\|_{op})^q,\qquad
\abs{\frac{\partial^2 F}{\partial A_{ij}\partial A_{kl}}(A)}\leq C(1+\|A\|_{op})^q
\]
for all $i,j,k,l\in\{1,\dots,d\}$ and any $d\times d$ symmetric matrix $A$, then 
\[
\sum_{i=1}^d\int_0^t\E[|\trace(\nabla F(\tilde V_s)H_{i,s})|^2]\inte s
+\int_0^t\E[|\trace\bra{F(\tilde V_s)\Gamma_s^2}|]\inte s
+\sum_{i=1}^d\int_0^t\E[|\langle \nabla^2F(\tilde V_s),\, H_{i,s}^{\otimes2}\rangle|]\inte s
<\infty.
\]
Hence, 
\[
\frac{\inte}{\inte t}\E[F(\tilde V_t)]=\E\sbra{\frac{1}{2\Lambda^2}\sum_{j=1}^d\langle \nabla^2 F(\tilde V_t),\, H_{j,t}^{\otimes2}\rangle-\frac{1}{\Lambda}\langle \nabla F(\tilde V_t),\, \Gamma_t^2\rangle}.
\]
In particular, we can apply this identity to $F(A)=\trace f(A)$ because $f$ is a $C^2$ function of polynomial growth. 
The remaining proof is the same as that of \cite[Lemma 5.3]{KlLe26}.  
\end{proof}

The next lemma is a variant of \cite[Lemma 5.4]{KlLe26}.
\begin{lemma}\label{kl26-lem5.4}
Let $f:\mathbb R\to[0,\infty)$ be a $C^2$ function satisfying the following conditions for some constants $D>1$ and $r\in[2,3]$:
\ben{
\begin{cases}
\text{$f$ is increasing},\\
f(x)=x^2\text{ for all }x\geq r,\\
f''(x)\leq D^2f(x)\text{ for all }x\geq0.
\end{cases}
}
Then for any $0<t<1$,
\[
\frac{\inte}{\inte t}\E[\trace f(\tilde V_t)]
\lesssim\frac{1}{(1-t)^2}\bra{\frac{1-t}{t}+D^2\sqrt{\Lambda\frac{1-t}{t}}}\E[\trace f(\tilde V_t)].
\]
\end{lemma}

\begin{proof}
Note that $f'$ is non-negative by assumption. 
Therefore, according to \cref{kl26-lem5.3}, it suffices to prove
\ba{
\sum_{i,j=1}^d|\xi_{ij}(t)|^2\frac{f'(\lambda_i(t))-f'(\lambda_j(t))}{\lambda_i(t)-\lambda_j(t)}
\lesssim\bra{\frac{1-t}{\Lambda t}+D^2\sqrt{\frac{1-t}{\Lambda t}}}\sum_{i=1}^df(\lambda_i(t)).
}
Noting that $\tilde{P}^*_t$ is $\frac{\Lambda t}{1-t}$-strongly log-concave, we can prove this inequality by a similar argument to the proof of \cite[Lemma 5.4]{KlLe26}. 
\end{proof}

Now we are ready to prove \eqref{v-fr-bound}. 
\begin{proof}[\bf\upshape Proof of \eqref{v-fr-bound}]
We follow the argument in the proof of \cite[Proposition 5.1]{KlLe26}. 
First, thanks to \eqref{v-bl-bound}, it suffices to consider the case $t<c\Lambda^{-1}$, where $c>0$ is a universal constant. 

Assume $t\leq2^{-8}\Lambda^{-1}$, and set $t_k:=2^{-8k} t$ and $D_k:=t_k^{-1/4}$ for every integer $k\geq0$. 
Note that $D_k>1$. 
In addition, define a sequence $(r_k)_{k=0}^\infty$ recursively by $r_0:=3$ and $r_{k+1}:=r_k-t_k^{1/8}$. 
By the proof of \cite[Proposition 5.1]{KlLe26}, we have $r_k\in[2,3]$ for all $k\geq0$. 
Hence, by the proof of \cite[Lemma 5.5]{KlLe26}, there exists an increasing $C^2$ function $f_k:\mathbb R\to(0,\infty)$ satisfying 
\[
f_k(x)=\begin{cases}
e^{D_k(x-r_k)} & \text{if }x\leq r_k-D_k^{-1},\\
x^2 & \text{if }x\geq r_k,
\end{cases}
\]
and $f''_k(x)\leq(12D_k)^2f_k(x)$ for all $x\geq0$. 
Further, for every $k\geq0$, observe that $D_k^4=t_k^{-1}\geq2^8\Lambda$. Hence, $D:=12\Lambda^{-1/4}D_k>1$ and for any $s\in[t_{k+1},t_{k}]$,
\[
D^2\sqrt{\frac{\Lambda}{s}}\leq144\frac{D_k^2}{\sqrt s}\leq\frac{144}{s}.
\] 
Therefore, applying \cref{kl26-lem5.4} to the function $f_k$ and $D=12\Lambda^{-1/4}D_k$ gives
\[
\frac{\inte}{\inte s}\E[\trace f(\tilde V_s)]
\lesssim\frac{1}{s}\E[\trace f(\tilde V_s)]
\quad\text{for any }s\in[t_{k+1},t_k].
\]
Then, similarly to the proof of Eq.(84) in \cite{KlLe26}, we deduce
\[
F_0\leq2^{Ck}F_k+d\sum_{i=0}^{k-1}2^{C(i+1)}\exp(-t_i^{-1/8})
\]
for any integer $k\geq0$, where $C>0$ is a universal constant and $F_k:=\E[\trace g_k(\tilde V_{t_k})]$ with $g_k(x)=x^21_{\{x\geq r_k\}}$. 
By the proof of \cite[Proposition 5.1]{KlLe26}, $t_i^{-1/8}\geq2(2^i-1)+t^{-1/8}$ for every $i$. 
Hence, noting that $2^{Ck}\leq t_k^{-C}$ and $\sum_{i=0}^{k-1}2^{C(i+1)}e^{-2(2^i-1)}\lesssim1$, we conclude 
\ben{\label{eq:F0-rec}
F_0\lesssim t_k^{-C}F_k+d\exp\bra{-t^{-1/8}}.
}
Now we prove $t_k^{-C}F_k\to0$ as $k\to\infty$. Since $r_k\geq2$,
\ba{
t_k^{-C}F_k
\leq t_k^{-C}\E\sbra{\sum_{i=1}^d\lambda_i(t_k)^21_{\{\lambda_i(t_k)\geq2\}}}
\leq \frac{d(1-t_k)}{\Lambda t_k^{C+1}}P\bra{\|\tilde V_{t_k}\|_{op}\geq2},
} 
where the last inequality follows from \eqref{v-bl-bound}. Since $P(\|\tilde V_{s}\|_{op}\geq2)\leq P(\|V_{s}\|_{op}\geq2\|\Sigma\|_{op})=s^q$ as $s\downarrow0$ for any $q>0$ by \cref{Gamma-tail}, we conclude $t_k^{-C}F_k\to0$ as $k\to\infty$. 
Therefore, letting $k\to\infty$ in \eqref{eq:F0-rec}, we obtain $F_0\lesssim d\exp(-t^{-1/8})\leq d$. 
Consequently,
\ba{
\E[\|\tilde V_t\|_F^2]=\sum_{i=1}^d\E[\lambda_i(t)^2]
\leq9d+\sum_{i=1}^d\E[g_0(\lambda_i(t))]
=9d+F_0
\lesssim d.
}
Hence $\E[\|V_t\|_F^2]=\Lambda^2\E[\|\tilde V_t\|_F^2]\lesssim\Lambda^2d$.
\end{proof}

\paragraph{Acknowledgments}

I thank Francesco Iafrate for teaching me the concept of weak convexity. 
I also thank Shogo Nakakita for letting me know \cite[Theorem 8]{VeWi23}. 
This work was partly supported by JST CREST Grant Number JPMJCR2115 and JSPS KAKENHI Grant Numbers JP24K14848, JP26K02870. 

{\small
\renewcommand*{\baselinestretch}{1}\selectfont

}

\end{document}